%% file: stobasket-paper.tex
\documentclass[sigconf,review=false]{acmart}
\setcopyright{none}
\acmConference[StoBasket Research Manuscript]{StoBasket Research Manuscript}{October 2026}{Preprint draft}
\acmYear{2026}\copyrightyear{2026}
\renewcommand\footnotetextcopyrightpermission[1]{}
\usepackage{booktabs,multirow,graphicx,amsmath,xcolor,array}
\newtheorem{proposition}{Proposition}

\newcommand{\1}{\mathbf 1}

\newcommand{\SB}{\textsc{StoBasket}}
\newcommand{\E}{\mathbb E}
\title[StoBasket]{StoBasket: Joint Upsell and Cross-Sell Optimization for Contribution Margin}
\author{Xin Heng}
\email{xin@tote.ai}
\affiliation{\institution{Tote AI}\country{United States}}
\begin{document}

\begin{abstract}
Retail recommenders rank relevance, but retailers realize contribution margin: an upsell earns the margin difference over a displaced item, whereas a cross-sell earns an incremental margin. We introduce \emph{joint stochastic basket editing}, a formulation in which one policy selects typed replacement and addition offers for an existing basket, values them by expected incremental contribution margin, and couples their values through distinct displayed-offer, accepted-offer, and joint-acceptance response channels. To our knowledge, it is the first formulation in which upsell and cross-sell responses are coupled: prior joint treatments assume this coupling away, which makes the two decisions separable, and are recovered here as the zero-coupling special case. This integrated action-level formulation is the central contribution of \SB{}: it places full-price recommendations, optional promotion terms, display and inventory constraints, and stochastic customer response in one decision. Real-data-calibrated simulations on five public basket datasets measure every policy against \emph{attainable offer margin}: the expected extra contribution margin that the best possible offers would earn under the simulated customer response (about \$0.47 per basket on Instacart). All percentages below are shares of it. With a correctly structured response model fit to a few thousand or more randomized offer sessions, \SB{} captures essentially all of it: within 0.5\% in every default-price setting of all five markets, including high display costs, and within 0.7\% when the true response lies outside its model family. Its remaining losses come from the response model, not the search: omitting cannibalization and offer fatigue costs 11--19\%. The economic layer is first-order: a margin-weighted popularity rule leaves as much as 11\% unrealized at ordinary display costs and up to 98\% when display is costly. On the three datasets that are also standard next-basket benchmarks, a published history-based recommender leaves 11--52\% unrealized even when reranked by margin, yet the same forecast adds up to 6\% when \SB{} uses it as an input. Joint selection pays most where premium tiers and add-ons are strongly linked: on Instacart it beats independent upsell and cross-sell selection by 10.2\% of attainable margin after empirical-Bayes noise correction and label placebos. Because coupling is measurable from a retailer's own baskets, the same analysis tells an operator where joint selection will pay, and active, randomized offer programs generate exactly the data that let it pay more broadly. When display friction makes actions profitable only together, the best heuristic portfolio loses 64\% under default Instacart prices while exact search stays within 0.5\%; across alternative price assumptions the heuristic loss ranges from 2\% to 89\%. For a factor-local subclass of the formulation, we characterize when staged policies succeed or fail and synthesize established geometric and graphical techniques into an exact optimizer. The two exact methods agree on all 8{,}040 checked real-basket instances. With up to 20 upsell groups and four displayed upsells, vectorized exact enumeration solves every tested real basket in 0.13\,ms at the median, 30\,ms at the 99th percentile, and at most 111\,ms in Python. Holding the optimizer fixed, a signal-free propensity model loses 38\%, while a correctly structured one trained on 5{,}000 randomized sessions comes within 0.06\% of perfect prediction: profit saturates long before accuracy does. Separately, randomized fashion slates show that a same-category neighbor causally lowers focal-item clicks by 12\%, supporting distinct exposure and acceptance channels. \SB{} therefore contributes the economic decision formulation first, evidence for when its coupling matters second, and a certifiable solver for the narrower regime in which local search fails.
\end{abstract}
\maketitle

\section{Introduction}\label{sec:intro}
Recommender systems commonly estimate product relevance or purchase propensity and then rank candidate items. In retail checkout, however, the decision concerns \emph{actions on an existing basket}. Replacing a product and adding a product have distinct financial consequences, feasibility constraints, and customer-response mechanisms. A high preference score does not, by itself, determine whether a recommendation increases the retailer's contribution margin.

Profit-aware recommendation provides a substantial foundation: surveys cover more than one hundred economic-recommender studies, and prior systems incorporate provider profit through reranking, joint relevance--profit objectives, or model training \cite{debiasio2024economic,louca2019profit,debiasio2024model}. Earlier deployed work also generated both upsell and cross-sell candidates for a current basket and ranked them by expected revenue or margin \cite{singh2006market}. \SB{} advances this line from economic item ranking to a joint stochastic edit policy whose replacement and addition actions have different incremental economics and whose values depend on what else is displayed and accepted.

Consider a basket containing regular burger patties and buns. Recommending premium patties is an \emph{upsell}: if accepted, it replaces the regular patties, so the financial gain is the premium margin \emph{minus} the margin of the displaced product. Recommending cheddar is a \emph{cross-sell}: if accepted, it adds an incremental item without removing the patties. The two decisions need not be independent. Recommending premium patties may change the probability of purchasing cheddar; conversely, the combined recommendation can be economically attractive even when neither component is profitable in isolation. Conventional matrix factorization can supply useful product representations, but a relevance score does not encode this upsell/cross-sell accounting or the interactions between recommended actions \cite{ruiz2020shopper,mcauley2015products}.

Our central contribution is the resulting \emph{decision formulation}. We formulate \emph{joint stochastic basket editing}: one policy chooses replacement and addition offers for an existing basket; replacement value deducts the margin of the displaced item, addition value does not; and the response law allows an offer's value to depend separately on which upsells are displayed, accepted, or accepted together. The policy then maximizes \emph{expected incremental contribution margin} over feasible offer sets. This combination is consequential: it makes the two recommendation types economically commensurable without making them behaviorally independent, and it creates profitable sets whose individual actions can all look unattractive. To our knowledge, no prior formulation couples the two action types. Earlier systems rank upsell and cross-sell candidates independently \cite{singh2006market} or assume responses under which the two subproblems separate \cite{ban2024substitutes}. Those models are the zero-coupling special case of ours: Proposition~\ref{prop:staging} shows that they remain optimal when coupling is absent and can fail arbitrarily badly when it is present.

\paragraph{Why a stochastic policy.} The word \emph{stochastic}, which gives \SB{} its name, refers to the customer, not to the search. The retailer chooses offers deterministically, but whether each offer is accepted is random, and the acceptances are correlated: a shopper who takes premium patties is more likely to take cheddar, and merely seeing one suggestion changes the response to another. A set of offers or recommendations therefore induces a probability distribution over possible final baskets, rather than determining a single final basket. Treating offers as if they were purchases overvalues every offer; scoring them one at a time ignores the correlation; and both errors grow exactly where interaction value is largest. The correct object is a policy that maximizes expected incremental margin under the joint response distribution.

That object raises a computational question. Enumerating offer sets grows exponentially with the number of candidates, ranking or staging offers carries no approximation guarantee once offers interact, and general-purpose mixed-integer or bandit search must either linearize an expectation over correlated choices or explore a combinatorial action space without certifying anything. \SB{} instead exploits linearity of expectation: expected margin depends on the response distribution only through first- and selected second-order acceptance moments (Eq.~\eqref{eq:mainmoments}), and every upsell choice influences cross-sells only through a short summary vector with one coordinate per response channel (Eq.~\eqref{eq:structured}). This compression supports both fast joint local search in ordinary markets and an exact optimizer where local search provably fails (Section~\ref{sec:algorithm}). The exact solution is also a trust anchor: for the same fitted objective, it certifies the greatest attainable predicted margin and so separates optimization error from prediction error. Recommendations may be shown at regular prices; item discounts and conditional bundle rebates are optional terms priced in the same objective, and the policy needs no persistent customer identifier.

Our experiments separate three sources of value. First, typed margin accounting matters in every evaluated market: even margin-weighted popularity leaves up to 11\% of attainable offer margin unrealized (Section~\ref{sec:empirical} defines this measure). Second, joint selection pays where premium tiers and add-ons are strongly linked, as in Instacart's organic and premium assortments, and Carbo-Loading's real prices independently confirm that such coupling exists; the same estimates yield a decision-aware market-basket diagnostic. Coupling is strongest in the most digitally active market, which we interpret as evidence that active offer programs both create and reveal the interactions \SB{} monetizes (Section~\ref{sec:impact}). Third, exact search matters in the narrower \emph{bundle-only-profitability} regime, where display friction makes each action unattractive alone although the set is profitable. Finally, profit saturates long before prediction accuracy does: with the optimizer fixed, a signal-free propensity model loses 38\% of attainable margin, while a correctly structured model trained on a few thousand randomized sessions comes within 0.06--0.3\% of perfect prediction, much as ranking skill rather than forecast precision drives value in portfolio investing.

\paragraph{Contributions.}
(1) The joint stochastic basket-editing formulation: replacement and addition actions with displaced-versus-incremental margin, feasibility constraints, display cost, and optional promotion terms in one expected-margin objective. (2) A compact response-to-value representation in which cross-sell value depends separately on displayed, accepted, and jointly accepted upsells, which shows exactly when separate pipelines are valid. (3) Evidence on five public basket datasets that economic accounting is broadly material, that supported coupling creates value in two markets, and that exact search is decisive only in the bundle-only regime. (4) For a factor-local, bounded-treewidth subclass, a characterization of when staged policies are exact, a proof that they have no constant-factor guarantee under coupling, and an exact solver verified on real interaction graphs. (5) A link from response estimation to decision regret, with a field-test and campaign design.

\begin{figure}[t]
\centering\includegraphics[width=0.95\columnwidth]{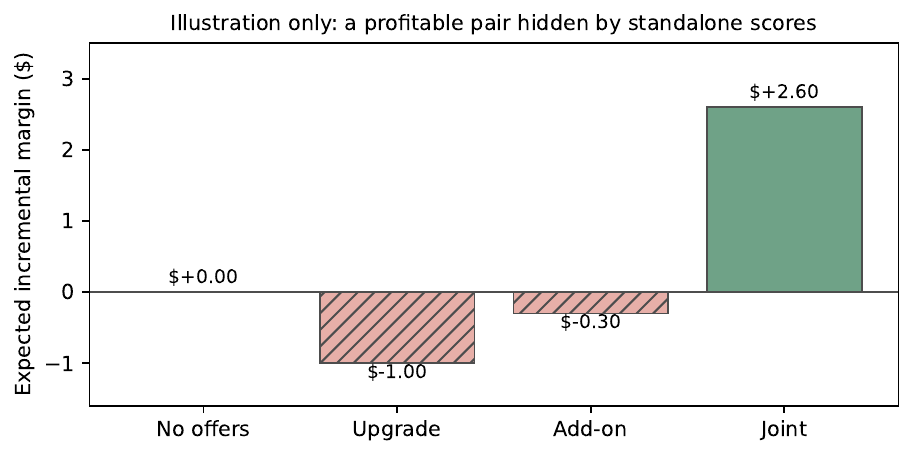}
\caption{\textbf{Illustration only} (constructed counterexample; not calibrated retail data). A deliberately chosen response law makes each singleton recommendation lose money while the joint set earns $+\$2.60$. The figure teaches the staging trap; it does not claim that typical grocery upgrades are unprofitable alone.}
\label{fig:example}
\end{figure}

\paragraph{Reading Figure~\ref{fig:example}.}
The instance is instructional. Upsell acceptance probability is $0.5$, but accepting the upsell \emph{reduces} incumbent-category margin by $\$2$ (expected $-\$1$ alone). That is not the usual ``cheap SKU $\to$ higher-margin premium''; it is the opposite case---for example a private-label to national-brand swap whose shelf price rises while the margin rate falls, or a promoted upgrade whose discount erodes the gain. The add-on (cheddar) has margin $\$10$ and base purchase probability $0.02$, so its standalone expected margin is $\$0.20$; displaying it costs $\$0.50$, hence $-\$0.30$ alone. Display cost is attached to the cross-sell and is the same whether or not the upsell is shown. Offering both raises cheddar purchase probability from $0.02$ to $0.80$ when the upsell is accepted, so the joint expected increment is $+\$2.60$. Every heuristic that requires a profitable first step therefore stops at ``show nothing,'' while the joint policy captures the interaction. Ordinary markets often have positive upgrades alone; the figure isolates the rarer bundle-only-profitability regime that Proposition~\ref{prop:staging} and the exact solver address.

\section{Recommendation System and Decision Process}\label{sec:pipeline}
\paragraph{Context and candidate actions.}
At a recommendation opportunity, the system observes a basket $B$ and context $s$, including product attributes, prices, contribution margins, inventory, store or region, and temporal features where available. We operationalize an \emph{upsell} as an upgrade that displaces an incumbent product, represented by \texttt{REPLACE(i,j)}, and a \emph{cross-sell} as an add-on that extends the basket, represented by \texttt{ADD(k)}. Compatibility rules prohibit conflicting upsells, while display and cardinality budgets restrict the number of recommendations. Contextual features condition the choice model; they are not additional combinatorial decision variables.

\paragraph{Choice estimation.}
The learning layer predicts which upsells shoppers accept, which cross-sells they buy, and how both depend on the combination shown. Estimating display effects requires logs of what was shown; purchase-only histories reveal product relationships but not responses to offers that were never made. For theory and simulation we use a finite latent-class response model: each shopping occasion has an unobserved type, such as ``premium-minded'' or ``value-minded,'' with its own acceptance probabilities, and averaging over types yields consistent individual and joint acceptance probabilities. Any estimator that supplies coherent probabilities can replace it (Section~\ref{sec:learningmain}; Appendix~\ref{app:model}).

\paragraph{Economic optimization and serving.}
A joint optimizer returns the feasible offer set with maximum expected incremental margin, choosing upsells and cross-sells together rather than in an upsell-first or cross-sell-first pipeline, and serves it through a checkout screen, cart page, or associate device. Realized revenue, displaced-product margin, promotion redemption, and display costs determine the outcome. Section~\ref{sec:formulation} defines the objective, Section~\ref{sec:algorithm} the optimizer, and Section~\ref{sec:experiments} the evaluation.

\section{Joint Stochastic Basket-Edit Formulation}\label{sec:formulation}
Four equations form the formulation: Eq.~\eqref{eq:policy} states the decision, Eq.~\eqref{eq:typedmargin} the accounting for one realized response, Eq.~\eqref{eq:structured} the value of \emph{every} candidate offer set in a structured subclass (the objective, not its solution), and Eq.~\eqref{eq:mainmoments} the response probabilities that evaluate it. Section~\ref{sec:algorithm} finds the maximizer.

\paragraph{Running example.} We illustrate every symbol with the burger basket of Section~\ref{sec:intro}. The basket contains regular patties and regular buns. The retailer can recommend two upsells, premium patties in place of the regular patties and brioche buns in place of the regular buns, and two cross-sells, cheddar and bacon. Per-unit contribution margins are \$1.50 for regular patties, \$3.00 for premium patties, and \$1.20 for cheddar. Table~\ref{tab:notation} collects the notation.

\begin{table}[t]
\caption{Notation, with the running burger example.}
\label{tab:notation}
\centering\footnotesize
\begin{tabular}{p{0.17\columnwidth}p{0.73\columnwidth}}
\toprule
Symbol & Meaning (example)\\\midrule
$B$ & Basket at decision time (regular patties, regular buns)\\
$s$ & Transaction context: prices, margins, inventory, store, time\\
$A$ & Offer set, or policy (premium patties and cheddar)\\
$D$ & Promotion terms; $D=\varnothing$ means ordinary prices\\
$\mathcal A(s;D)$ & Permitted offer sets under budgets, conflicts, and stock\\
$Y$ & Random shopper response: which offers are accepted\\
$B\oplus Y$ & Final basket after accepted edits are applied\\
$M_D,\ M_0$ & Contribution margin of a basket after promotions; of the original basket $B$\\
$F$ & Objective: expected incremental contribution margin of an offer set, net of display cost\\
$c(A,D;s)$ & Cost of presenting $A$; written $c(A;s)$ when $D=\varnothing$\\
$m_i$ & Per-unit contribution margin of product $i$\\
$g,\ G$ & Upsell group (incumbent item) and number of groups\\
$x_g,\ J_g$ & Upsell offered in group $g$ ($0$ = none) and number of alternatives\\
$C,\ M$ & Offered cross-sells and number of candidates\\
$L,\ K$ & Maximum numbers of upsells and cross-sells shown\\
$Y_g,\ Z_k$ & Indicators: upsell $g$ accepted; add-on $k$ purchased\\
$p_g,\ \rho_{gh}$ & Probability upsell $g$ is accepted; both $g$ and $h$ accepted\\
$z,\ r$ & Response summary of the upsell choice; its dimension\\
$U$ & Expected value of the upsells alone\\
$\alpha_k,\ \beta_k$ & Standalone value of add-on $k$; its sensitivity to $z$\\
$\eta_{k\ell}$ & Interaction between add-ons $k$ and $\ell$ shown together\\
\bottomrule
\end{tabular}
\end{table}

\paragraph{The decision.} For any permitted offer set $A$, define its value
\begin{align}\label{eq:policy}
F(A;s,D)={}&\E[M_D(B\oplus Y)\mid s,A,D]-M_0(B;s)-c(A,D;s),\notag\\
A^*(D)\in{}&\arg\max_{A\in\mathcal A(s;D)}F(A;s,D).
\end{align}
In words: $F$ scores a candidate offer set by how much its expected final-basket contribution margin exceeds the margin of the basket as it stands, net of the cost of making the offers; the policy chooses the permitted offer set with the highest score. Thus $M$ measures a basket, while $F$ measures an offer set. Each component has a concrete meaning.
\begin{itemize}
\item $A$ is the offer set. Recommending premium patties in place of regular patties together with cheddar is $A=\{\texttt{REPLACE}(i,j),\ \texttt{ADD}(k)\}$, with $i$ the regular patties, $j$ the premium patties, and $k$ cheddar.
\item $\mathcal A(s;D)$ is the set of permitted offer sets: at most $L$ upsells and $K$ cross-sells, at most one replacement per basket item, in-stock products only, and any merchandising rules.
\item $Y$ is the shopper's random response to $A$, for example ``accepts premium patties, declines cheddar.'' $B\oplus Y$ is the resulting final basket, here premium patties and regular buns.
\item $M_D(\cdot)$ is the contribution margin of a basket---revenue minus variable product cost---after any redeemed promotion, and $\E[\,\cdot\mid s,A,D]$ averages it over the possible responses $Y$. The subscript in $M_0(B;s)$ is zero: it is the contribution margin of the original basket before any offer is made. Subtracting it turns the expectation into an \emph{incremental} contribution margin.
\item $c(A,D;s)$ collects the costs of making the offers that $M_D$ does not already include: screen space, associate time, and shopper attention. When $D=\varnothing$ we write $c(A;s)$.
\end{itemize}
The same objective covers ordinary recommendations ($D=\varnothing$), fixed item discounts, which change realized margins and response probabilities, and conditional bundle rebates: a rebate $r$ paid only when upgrade $g$ and add-on $k$ are both accepted contributes $-r\Pr(Y_g=1,Z_k=1\mid s,A,D)$. We select offers under given promotion terms; optimizing discount levels is a distinct extension (Appendix~\ref{sec:promotion}).

\paragraph{Typed economic effects.}
Consider one upsell in group $g$ that replaces incumbent product $i$ with alternative $j$, and one cross-sell add-on $k$. Let $Y_g=1$ if the shopper accepts the upsell and $Y_g=0$ otherwise, and let $Z_k=1$ if the shopper buys the add-on. At ordinary prices, the realized margin increment inside Eq.~\eqref{eq:policy} is
\begin{equation}\label{eq:typedmargin}
\Delta M=Y_g(m_j-m_i)+Z_km_k-c(A;s).
\end{equation}
The first term is the replacement: an accepted upsell earns the \emph{difference} in margin, because the incumbent leaves the basket. Premium patties replacing regular patties earn $\$3.00-\$1.50=\$1.50$, not \$3.00. The second term is the addition: a purchased add-on earns its full margin, \$1.20 for cheddar, because nothing is displaced. The third term is the cost of presenting the offers. Taking expectations gives $\E[\Delta M]=p_g(m_j-m_i)+\Pr(Z_k=1)\,m_k-c(A;s)$, with $p_g$ defined in Eq.~\eqref{eq:mainmoments}. Because the replacement/addition type is part of the action, it changes the value optimized \emph{before} serving rather than appearing only in a report afterward. When the recommendations interact, the joint distribution of $(Y_g,Z_k)$ depends on the selected policy, so optimizing the two action types independently may be suboptimal.

\paragraph{Core structured coupling representation.}
Let the basket have $G$ upsell groups, one per incumbent item that could be upgraded; in the example $G=2$ (patties and buns). Let $x_g\in\{0,\ldots,J_g\}$ denote the upsell recommended for group $g$, where $0$ means no upsell and $J_g$ is the number of alternatives; $x_1=1$ recommends premium patties. Let $C$ be the set of recommended cross-sells, chosen from $M$ candidates; $C=\{\text{cheddar}\}$ in the example. Every offer set $A$ in this subclass is therefore represented by a pair $(x,C)$. Fix the observed context $s$ and the promotion terms $D$, and absorb those terms into the margins and response parameters. For \emph{any} feasible pair $(x,C)$, write $F(x,C;s)$ for the value $F(A;s,D)$ of Eq.~\eqref{eq:policy} when $A$ is the offer set described by $(x,C)$. Under the structured response assumptions of Appendix~\ref{app:model}, this value takes the form
\begin{align}\label{eq:structured}
F(x,C;s)={}&U(x;s)+\sum_{k\in C}\big[\alpha_k(s)+\beta_k(s)^\top z(x;s)\big]\notag\\
&+\sum_{\{k,\ell\}\subseteq C}\eta_{k\ell}(z(x;s);s),\\[-3pt]
z(x;s)={}&\sum_{f\in\mathcal H}\mu_f(x_{S_f};s).\nonumber
\end{align}
Equation~\eqref{eq:structured} assigns a value to every candidate offer set, profitable or not. The optimal offer set is the candidate with the largest value:
\begin{equation}\label{eq:structuredopt}
(x^*,C^*)\in\arg\max_{(x,C)\in\mathcal A(s;D)}F(x,C;s).
\end{equation}
Equation~\eqref{eq:structuredopt} is Eq.~\eqref{eq:policy} restricted to the structured subclass. Section~\ref{sec:algorithm} solves it exactly, and Appendix~\ref{app:model} derives Eq.~\eqref{eq:structured} from an affine customer-response model (Appendix Eq.~\eqref{eq:objective}). The terms of Eq.~\eqref{eq:structured} read as follows in the burger basket.
\begin{itemize}
\item $U(x;s)$ is the expected value of the upsells by themselves. For each offered upsell it adds the acceptance probability times the margin difference, minus the cost of showing it, plus any extra margin when two upsells are accepted together. If premium patties are accepted with probability $0.5$ and cost \$0.10 to show, they contribute $0.5\times\$1.50-\$0.10=\$0.65$.
\item $\alpha_k(s)$ is the standalone value of add-on $k$: its margin times its purchase probability when no upsell is shown, minus its display cost. If cheddar sells with probability $0.10$ on its own and costs \$0.05 to show, $\alpha_{\text{cheddar}}=0.10\times\$1.20-\$0.05=\$0.07$.
\item $z(x;s)\in\mathbb R^r$ summarizes, in $r$ numbers, everything about the upsell choice $x$ that matters to add-ons. In the response model of Appendix~\ref{app:model}, $r=3$, one coordinate per channel through which upsells change add-on purchases: (i) \emph{accepted upsells}, $\sum_gp_gt_g$, where $t_g$ measures how much accepting upsell $g$ shifts add-on demand: a shopper who switches to premium patties is more likely to add cheddar; (ii) \emph{displayed upsells}, $\sum_gv_g$, where $v_g$ measures the effect of merely showing upsell $g$: seeing the premium suggestion changes the cheddar decision even if the shopper declines it; and (iii) \emph{joint acceptance}, $\sum_{(g,h)}\rho_{gh}\zeta_{gh}$, where $\zeta_{gh}$ measures the extra shift when upsells $g$ and $h$ are both accepted: premium patties together with brioche buns signal a ``better burger'' beyond either alone.
\item $\beta_k(s)\in\mathbb R^r$ converts that summary into money for add-on $k$: each coordinate is the add-on's margin times its sensitivity to the corresponding channel. Thus $\beta_k^\top z$ is how much the chosen upsells raise or lower add-on $k$'s expected value, and $\alpha_k+\beta_k^\top z$ is its total contribution when shown.
\item $\eta_{k\ell}$ is the optional interaction between two add-ons shown together. Cheddar and bacon may compete for the same burger spend ($\eta<0$, cannibalization) or reinforce each other ($\eta>0$); the interaction may itself depend on $z$.
\item The second line states the general structure of $z$: a sum of local contributions $\mu_f$, each depending only on the upsell choices $x_{S_f}$ in a small set of groups $S_f$. The collection $\mathcal H$ of such factors includes one factor per group, contributing $(p_gt_g,v_g,0)$, and one per correlated pair, contributing $(0,0,\rho_{gh}\zeta_{gh})$. Locality is what makes exact optimization possible in Section~\ref{sec:algorithm}.
\end{itemize}
These channels are distinct economically and statistically: display can change response without acceptance, and correlated acceptance can create value not recoverable from marginal probabilities. The coupling term $\beta_k^\top z(x;s)$ is what distinguishes this formulation from prior joint treatments. When it is zero for every feasible $x$ and $k$, Eq.~\eqref{eq:structured} separates into an upsell problem and a cross-sell problem (Proposition~\ref{prop:separability}), which recovers the separable models of earlier work \cite{ban2024substitutes}.

The dimension $r$ counts response channels, not products, promotions, or action types: a discount or bundle rebate changes coefficients, not $r$ (Appendix~\ref{sec:promotion}), and a fourth coordinate would add another channel, for example separate meat and bread effects. The summary has a dual role: it makes the policy nonseparable while compressing many upsell choices into a fixed-dimensional economic state, which is what Section~\ref{sec:algorithm} exploits. The number of distinct add-on rankings the optimizer must consider grows as $M^{O(r)}$, polynomial in the number of candidates for fixed $r$.

\paragraph{Coherent customer-response moments.}
For an offered upgrade $g$, let $Y_g$ denote acceptance, and define
\begin{equation}\label{eq:mainmoments}
p_g=\Pr(Y_g=1\mid A,s),\qquad
\rho_{gh}=\Pr(Y_g=Y_h=1\mid A,s).
\end{equation}
Here $p_g$ is the probability that the shopper accepts upsell $g$ given the whole offer set, and $\rho_{gh}$ the probability of accepting both $g$ and $h$; they weight the margin terms in $U$ and the accepted and joint-acceptance channels of $z$. In the example, $p_1=p_2=0.5$ while $\rho_{12}=0.30$ exceeds the $0.25$ that independence would imply, because shoppers who upgrade one item tend to upgrade the other. These first- and second-order moments are the interface between any coherent response model---neural, latent-variable, or otherwise---and the economic objective; they must come from a valid choice distribution and respect the factor scopes the optimizer requires (Appendix~\ref{app:model}). For one upsell and one cross-sell, the mixed difference $F(1,1)-F(1,0)-F(0,1)+F(0,0)$ of Eq.~\eqref{eq:structured}, where the arguments indicate whether each offer is shown, measures the value of coupling them; in Figure~\ref{fig:example} it is $\$2.60+\$1.00+\$0.30=\$3.90$.

\section{Exact Structured Optimization}\label{sec:algorithm}
This section solves the maximization of Eq.~\eqref{eq:structuredopt}: given the value $F(x,C;s)$ of every candidate offer set, find the candidate with the largest value. Combining hyperplane arrangements with tree-decomposition dynamic programming \cite{lo2019synergy,delpia2025factor}, the multi-channel structure of Eq.~\eqref{eq:structured} identifies a nonseparable class of stochastic basket-edit policies whose global optimum is computable. The main text gives the algorithm and its intuition; Appendix~\ref{app:proofs} gives the proof, complexity, tie handling, and implementation.

The central computational obstacle is the exponential number of joint recommendation sets: trying every combination of upsells and cross-sells quickly becomes infeasible. Equation~\eqref{eq:structured} exposes structure that permits a more efficient search. In the additive cross-sell subclass of Eq.~\eqref{eq:structured} ($\eta_{k\ell}=0$), each candidate add-on has an affine score $\ell_k(z)=\alpha_k+\beta_k^\top z$, its total contribution when shown. Given a fixed upsell summary $z$ and a cross-sell budget $K$, the optimal cross-sells are the $K$ largest positive scores. The ordering of these scores changes only across pairwise-equality and zero-score hyperplanes in the $r$-dimensional summary space.

This observation yields a finite family of candidate cross-sell sets. For each candidate set, the remaining upsell objective is a sum of the local factors of Eq.~\eqref{eq:structured} and can be optimized by a count-augmented dynamic program on a tree decomposition of the upsell interaction graph, the graph that joins two upsell groups whenever they share a factor scope $S_f$. The resulting algorithm combines \emph{geometric enumeration of economically relevant cross-sells} with \emph{graphical optimization of feasible upsells}, avoiding enumeration of every possible joint basket edit.

\begin{proposition}[Exact joint optimization for structured responses]\label{prop:mainexact}
Suppose that in Eq.~\eqref{eq:structured} cross-sell contributions are additive ($\eta_{k\ell}=0$) and affine in a summary $z$ of fixed dimension $r$; the upgrade factors in $U$ and $z$ have a primal interaction graph of fixed treewidth; all input coefficients are rational; and upgrade and add-on budgets $L$ and $K$ are separate. Under coherent, policy-local acceptance moments (Eq.~\eqref{eq:mainmoments}), an offer set $(x,C)$ maximizing Eq.~\eqref{eq:structured}, and hence the expected incremental margin of Eq.~\eqref{eq:policy} within this class, can be computed in polynomial time in the encoded input size.
\end{proposition}

\noindent\textbf{Algorithm.} (1) Estimate the response moments $p_g$ and $\rho_{gh}$ of Eq.~\eqref{eq:mainmoments} for context $s$. (2) Construct the terms $U$, $\alpha_k$, $\beta_k$, and $z$ of Eq.~\eqref{eq:structured}. (3) Enumerate the candidate top-$K$ cross-sell sets induced by the arrangement of scores $\ell_k(z)=\alpha_k+\beta_k^\top z$, including boundary ties. (4) For each candidate set $C$, maximize Eq.~\eqref{eq:structured} over upsells $x$ by tree-decomposition dynamic programming subject to the upsell budget $L$. (5) Return the pair $(x,C)$ with the largest value of Eq.~\eqref{eq:structured}; this is the maximizer $(x^*,C^*)$ of Eq.~\eqref{eq:structuredopt}.

\paragraph{What the certificate means operationally.}
Under the assumptions of Proposition~\ref{prop:mainexact}, no other feasible offer set has greater value under the supplied response estimates and economic objective. A retailer may therefore run any faster heuristic it prefers and use the exact policy selectively or offline as a trust anchor: the difference $F(x^*,C^*;s)-F(\widehat x,\widehat C;s)$ measures the heuristic's optimization loss on the same basket and model. Exact execution makes this search loss zero and isolates the remaining uncertainty in the response estimates, margins, and structural assumptions. It does \emph{not} make poor data accurate or an omitted customer-response mechanism harmless; Sections~\ref{sec:learningmain} and \ref{sec:limitations} quantify those separate prediction and specification risks.

The proof and exact assumptions appear in Appendix~\ref{app:proofs} (Proposition~\ref{prop:poly}). Because upgrade factors may span several groups, the statement also covers local exposure effects in which displaying one upsell changes another's acceptance. Add-on interactions ($\eta_{k\ell}\neq0$) lie outside this exact class; without structure on them, the problem is NP-hard (Appendix~\ref{app:hardness}). The implementation samples every arrangement cell in exact rational arithmetic and runs the dynamic program on a min-fill tree decomposition; it matches exhaustive search on cyclic graphs of induced width up to four for $r=3$ and $r=4$ (Appendix~\ref{sec:twexp}). The arrangement has $M^{O(r)}$ cells, so bounded $r$ is a substantive modeling assumption.

\paragraph{Five heuristic solution strategies.}
Practical systems often avoid the joint problem by staging it. We compare the exact optimizer with five heuristics that solve the \emph{same} problem, Eq.~\eqref{eq:structured}, with the same objective and the same exact sub-solvers for the best upsells given cross-sells ($x\mid C$) and the best cross-sells given upsells ($C\mid x$). They differ only in how they search, not in their customer model (Table~\ref{tab:strategies}).

\begin{table*}[t]
\caption{The five heuristic strategies and the exact optimizer, all solving Eq.~\eqref{eq:structured} with the same objective and sub-solvers. ``Trap'' refers to bundle-only profitability (Figure~\ref{fig:example}; Proposition~\ref{prop:staging}).}
\label{tab:strategies}
\centering\footnotesize
\begin{tabular}{>{\raggedright\arraybackslash}p{0.15\textwidth}>{\raggedright\arraybackslash}p{0.36\textwidth}>{\raggedright\arraybackslash}p{0.41\textwidth}}
\toprule
Strategy & How it chooses offers & Strengths and weaknesses\\\midrule
1. Independent & Picks upsells as if no cross-sell were shown (maximizes $U(x)$) and cross-sells as if no upsell were shown (maximizes $\sum_{k\in C}\alpha_k$). & $+$ Fastest; two small separate problems; optimal when offers do not interact. $-$ Ignores all cross-action coupling; largest gap when coupling is strong.\\
2. Upsell-first & Keeps the independent upsells, then picks the best cross-sells given their summary $z(x)$. & $+$ Cheap; captures the effect of upsells on cross-sells. $-$ Upsells are chosen blind to cross-sells, and a poor first stage is never revised.\\
3. Cross-sell-first & Keeps the independent cross-sells, then picks the best upsells given them. & $+$ Cheap; picks upsells that raise the value of the chosen cross-sells. $-$ Cross-sells are chosen blind to upsells; no revision.\\
4. Alternating best response & Starts from the better of strategies 2 and 3, then alternates $x\mid C$ and $C\mid x$ until neither step improves. & $+$ Fast and usually near-exact; within 0.14\% of exact on Carbo-Loading. $-$ Stops at a local optimum; can stay at ``show nothing'' in the trap; no worst-case guarantee.\\
5. Greedy & Repeatedly adds the single offer with the largest positive gain until no offer helps. & $+$ Simple; needs no structure in the objective. $-$ Myopic: requires a profitable first step, so it fails in the trap and when pairs interact.\\
Exact (\SB{}) & Enumerates candidate cross-sell sets by score geometry and optimizes upsells by tree-decomposition dynamic programming (Proposition~\ref{prop:mainexact}). & $+$ Certified global optimum; captures profitable sets with unattractive parts. $-$ Needs fixed $r$ and bounded treewidth; more compute (0.13\,ms median, 111\,ms maximum on real baskets).\\
\bottomrule
\end{tabular}
\end{table*}

\begin{proposition}[Staging gap]\label{prop:staging}
(i) Independent $\le$ upsell-first, cross-sell-first $\le$ alternating $\le$ optimum, for every instance. (ii) If $\beta_k^\top z(x)=0$ for all feasible $x$ and $k$, and the budgets are separate, then independent, both sequential policies, and alternating best response are optimal; greedy need not be when upsell or add-on pairs interact. (iii) For every $\lambda>0$ there is an instance with one upsell and one cross-sell on which all five policies return value $0$ while the optimum is $2.6\lambda$. Hence none of them has a positive approximation ratio. (iv) For every $n\ge1$ there is a valid instance with $n$ upsells and one cross-sell in which every nonempty policy with at most $n$ offers has negative value, while offering all $n+1$ is profitable. Local search that changes up to $n$ offers per step, started from the empty policy, therefore also returns $0$.
\end{proposition}
\noindent\emph{Proof sketch.} (i) Each later policy optimizes over a set containing the earlier policy's choice. (ii) Under the condition, $F(x,C)=U(x)+\sum_{k\in C}\alpha_k$ by Proposition~\ref{prop:separability}, so the independent policy is optimal, and (i) closes the chain. (iii) Scale the instance of Figure~\ref{fig:example} by $\lambda$. From the empty policy, each single action has negative value, so each staged step, each best response, and each greedy move retains the empty policy. (iv) Each upsell loses a little alone but gains $b>0$ when the cross-sell is shown; the cross-sell alone loses $A$ with $(n-1)b<A<nb$. Appendix~\ref{app:staging} gives the explicit valid response model.

\section{Learning Customer Choice and Decision Quality}\label{sec:learningmain}
A finite latent-class model provides one coherent construction for the moments of Eq.~\eqref{eq:mainmoments}. Let a latent preference class $Q$ take value $q$ with probability $\pi_q$, and let $p_{qg}$ denote the acceptance probability for upgrade $g$ conditional on class $q$ and its exposure. Under conditional independence within a class,
\begin{equation}\label{eq:latentclass}
p_g=\sum_q\pi_q p_{qg},\qquad
\rho_{gh}=\sum_q\pi_q p_{qg}p_{qh}.
\end{equation}
In the burger example, one class might be ``premium-minded'' shoppers who accept both upgrades with high probability and another ``value-minded'' shoppers who accept neither; mixing the two produces $\rho_{12}>p_1p_2$ even though choices are independent within each class. This latent-class construction gives a concrete route to mutually coherent moments. The mixture induces correlated decisions after marginalizing $Q$. Its class weights and acceptance rates can be conditioned on the basket, store, and time; individual purchase history can provide additional features. The optimizer requires coherent response moments, rather than persistent identity information or a particular neural architecture.

Choice estimation affects decisions through errors in expected policy values. Let $F(A)$ denote the true value of offer set $A$ (Eq.~\eqref{eq:policy}, or Eq.~\eqref{eq:structured} in the structured class), and let $\widehat F$ be the value function implied by estimated moments, with $\sup_{A\in\mathcal A}|F(A)-\widehat F(A)|\leq\varepsilon$. If $\widehat A$ maximizes $\widehat F$, the standard plug-in decision inequality gives
\begin{equation}\label{eq:regret}
F(A^*)-F(\widehat A)\leq 2\varepsilon.
\end{equation}
Section~\ref{sec:insights} measures this link directly by varying predictor quality while holding the optimizer fixed (Figure~\ref{fig:prediction}), complementing existing choice-learning and decision-regret analyses \cite{erginbas2025online}.

\section{Economic Results and Evaluation}\label{sec:experiments}
\paragraph{Coupling multiplier.} \emph{Cross-action coupling} is the term $\beta_k^\top z(x)$ through which upsells change a cross-sell's value. In the experiments a dimensionless multiplier $\gamma\ge0$ scales only these terms, holding margins, base rates, display costs, candidates, and budgets fixed: $\gamma=0$ removes coupling, $\gamma=1$ is the generated or receipt-calibrated reference strength, and $\gamma=2$ is a doubled stress test. It is not a probability. If an upsell raises an add-on's purchase probability from 8\% to 12\% at $\gamma=1$, the same channel gives 8\%, 10\%, and 16\% at $\gamma=0$, $0.5$, and $2$.

\begin{figure*}[t]
\centering\includegraphics[width=0.85\textwidth]{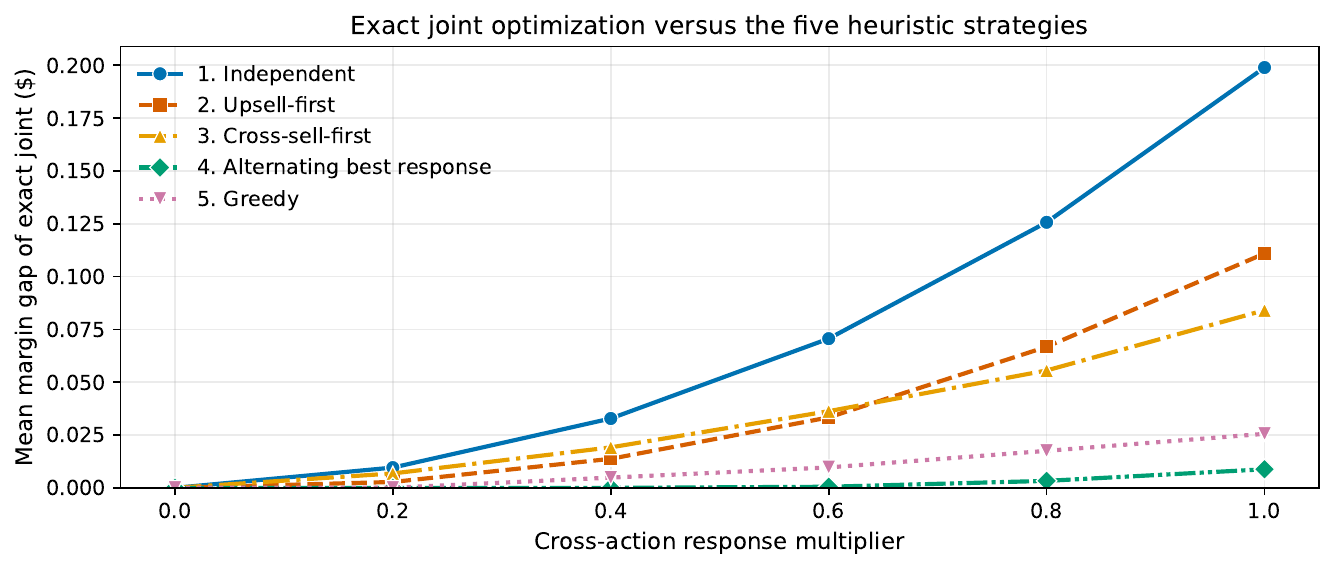}
\caption{Mean expected-margin gap between the exact joint optimum and each of the five heuristic strategies of Table~\ref{tab:strategies} as the multiplier $\gamma$ scales all generated upsell-to-cross-sell response terms. The same 60 synthetic instances are evaluated at each multiplier; $\gamma=0$ removes cross-action coupling and $\gamma=1$ restores the generated reference strength. All staged strategies are exact at zero coupling.}
\label{fig:sweep}
\end{figure*}
\paragraph{Synthetic optimization experiments.}
In the 60-seed interaction sweep, every staged baseline and alternating best response coincides with the joint optimum at zero coupling, as Proposition~\ref{prop:staging}(ii) requires. The gaps grow with the cross-action effect (Figure~\ref{fig:sweep}). At full coupling, the mean advantage of the exact joint policy, and the number of instances on which it is strictly better, is $\$0.199$ (44/60) over independent, $\$0.111$ (32/60) over upgrade-first, $\$0.084$ (22/60) over add-on-first, $\$0.026$ (10/60) over greedy, and $\$0.009$ (6/60) over alternating best response. Best-response iteration recovers most of the independent gap but, by Proposition~\ref{prop:staging}(iii), carries no worst-case guarantee. The exact solver also matches exhaustive optimization on all reported small structured instances, including cyclic upgrade graphs (Appendix Table~\ref{tab:checks}).

\subsection{Real-data-calibrated evaluation}\label{sec:empirical}
No public dataset records which recommendations were displayed and accepted, so we evaluate by \emph{real-data-calibrated simulation} (Appendix~\ref{app:empirical}). Five public retail datasets provide real baskets and market structure. \emph{Complete Journey} supplies prices, discounts, stores and timestamps; margins are assumed. \emph{Ta Feng} supplies prices and per-line cost, hence real margins. \emph{Instacart} supplies large-scale in-cart sequences but requires price imputation. \emph{Online Retail II} is a non-grocery UK giftware replication. \emph{Carbo-Loading} supplies paid prices and built-in complementary categories---pasta and sauce, pancake mix and syrup---plus brand tiers and store promotion flags. Real test-period baskets define the decisions: up to four upsell groups with two premium alternatives each, six add-on candidates, and budgets $L=K=2$. Training receipts calibrate acceptance shares and the three channels of Section~\ref{sec:formulation}. Because true offer responses are unobserved, each experiment fixes one of three simulated \emph{response truths}, which determine how shoppers actually react and against which every policy is scored. \textbf{S1} (in-family) follows the response model of Section~\ref{sec:formulation}. \textbf{S2} (out-of-family) uses logistic responses and latent shopper classes that the fitted model cannot represent exactly. \textbf{S3} (stress) adds add-on cannibalization, offer fatigue, and a train--test shift. Appendix~\ref{app:empirical} gives their full specifications. Policies are learned from simulated randomized offer logs and scored by exact expected margin under the truth. We report incremental contribution margin per 1{,}000 test baskets with paired bootstrap intervals. Appendix~\ref{app:empirical} describes each dataset and lists its source and license.

\begin{table*}[t]
\caption{\textbf{Economic scorecard across five markets} (S1, the in-family response truth of Section~\ref{sec:empirical}; percentage of oracle offer margin). Economic-policy gain is the learned exact policy over margin-weighted popularity in the base regime; it combines response learning and decision optimization rather than isolating accounting alone. Joint-selection gain uses EB shrinkage where available. Exact-search premium is the largest portfolio gap in the tested display-friction map through $\tau=3$ for Complete Journey/Ta Feng and $\tau=2$ for Instacart/Carbo-Loading; Online Retail II has only the base comparison. $^\dagger$: 63.8\% uses EB shrinkage and default imputed prices; alternative price assumptions give 2--89\%.}
\centering\footnotesize
\input{tables/empirical_economic_summary.tex}
\label{tab:economic-summary}
\end{table*}

\begin{table*}[t]
\caption{Exact joint policy minus each baseline, in expected incremental margin per 1{,}000 baskets (USD; NT\$ for Ta Feng; Instacart prices imputed). Oracle and Exact are absolute values. All model-based policies use the same fitted model. $^\dagger$: 95\% interval contains zero. $^-$: baseline significantly better.}
\centering\footnotesize
\input{tables/empirical_e1.tex}
\label{tab:empirical}
\end{table*}

\begin{figure*}[t]
\centering\includegraphics[width=0.92\textwidth]{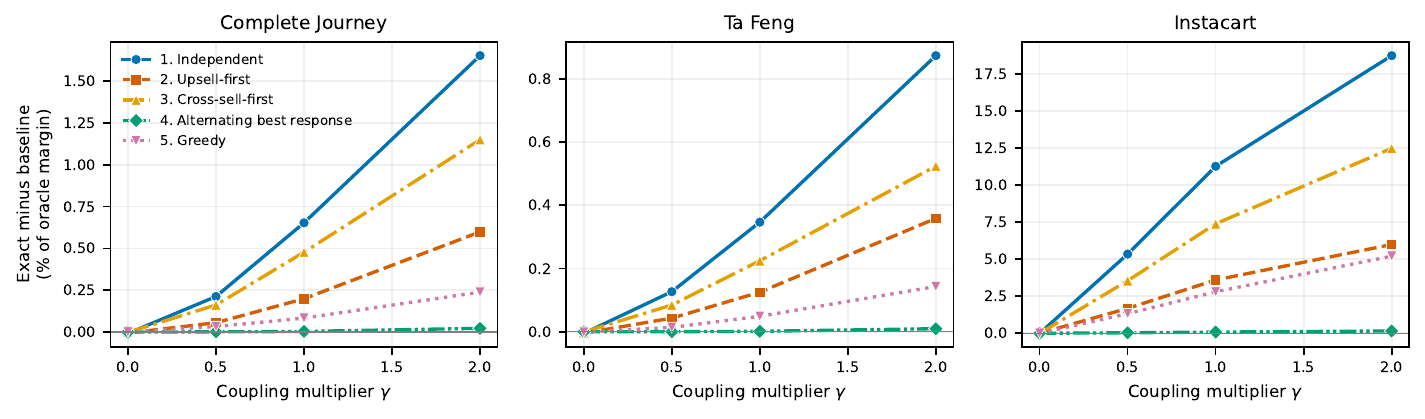}
\caption{Exact joint policy minus each baseline as the receipt-calibrated cross-action channel effects are scaled by $\gamma$ (in-family truth S1, Section~\ref{sec:empirical}), as a percentage of the oracle's expected margin. Here $\gamma=0$ removes the calibrated spillovers, $\gamma=0.5$ halves them, $\gamma=1$ is the reference calibration, and $\gamma=2$ doubles them as a stress test.}
\label{fig:empirical}
\end{figure*}

\begin{figure*}[t]
\centering\includegraphics[width=0.98\textwidth]{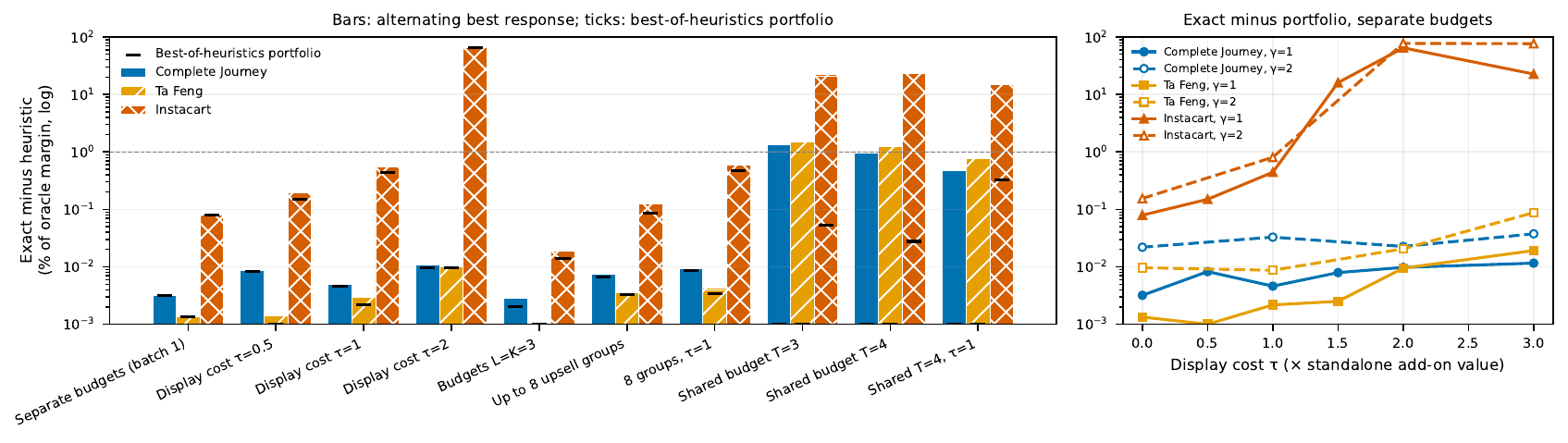}
\caption{\textbf{Where exact optimization matters.} Left: exact minus alternating best response (bars) and a post-hoc best-of-heuristics portfolio (ticks), as a percentage of oracle offer margin. Right: exact minus the portfolio under separate budgets as display cost increases; $\tau$ is display cost divided by mean standalone add-on value and $\gamma$ scales cross-action coupling. Weakly coupled Complete Journey and Ta Feng remain easy. Instacart enters a bundle-only-profitability regime in which all tested heuristics fail.}
\label{fig:whereexact}
\end{figure*}

\paragraph{Economic accounting is first-order.} Table~\ref{tab:economic-summary} separates the economic-policy gap over a margin-weighted popularity rule, the joint-selection gain, and the exact-search premium across all five markets; Table~\ref{tab:empirical} gives the primary three-dataset baseline detail, and Appendix Table~\ref{tab:moredata} gives the external-data rows. Unless stated otherwise, percentages are shares of \emph{attainable offer margin}: the oracle policy's expected incremental margin under the simulated truth, not total store profit. On Instacart this is about \$0.47 per basket, so leaving 11\% unrealized forgoes about 5 cents per basket; equivalently, \SB{} earns about 13\% more than that baseline. Against this benchmark, the learned \SB{} policy comes within 0.1\% of the oracle in every market when the simulated response matches its model family, and within 0.7\% when it does not (Table~\ref{tab:empirical}; Appendix Table~\ref{tab:moredata}). A relevance-style association rule that ignores displaced and incremental margins forfeits \$534 of \$536 per 1{,}000 Complete Journey baskets, 52\% of attainable margin on Ta Feng, 19\% on Instacart, and 15\% on Carbo-Loading. Weighting popularity by margin recovers most of that loss but still trails the learned economic policy by 2.4\%, 1.5\%, 11.4\%, 1.0\%, and 4.7\% on Complete Journey, Ta Feng, Instacart, Online Retail II, and Carbo-Loading, respectively. On Carbo-Loading the gap reaches 27\%, 54\%, and 98\% as display cost rises from one to two times standalone add-on value. Correctly valuing replacement versus addition is therefore the broad economic result; interaction-aware optimization determines how much more can be recovered.

\begin{table*}[t]
\caption{\textbf{Against prior art on shared benchmarks} (E20; percentage of oracle offer margin, 3{,}000 test baskets per row). TIFU-KNN \cite{hu2020tifuknn} is run from the reference code of the next-basket benchmark study \cite{li2023nbrcheck} with its published hyperparameters. $\omega$ scales a personal-habit effect on add-on purchase rates; $\omega=0$ is the reference in-family S1 truth (Section~\ref{sec:empirical}). \emph{StoBasket + TIFU} adds the TIFU-KNN forecast as one regressor in the same response model; \emph{relevance} fills every offer slot with TIFU-KNN's top-scored items; \emph{$\times$ margin} reranks them by score times margin gain. Bold: best at the shown precision.}
\centering\footnotesize
\input{tables/empirical_prior_art.tex}
\label{tab:priorart}
\end{table*}

\paragraph{Against prior art on the same benchmarks.} The baselines form three tiers, all applied to the same candidates and economics. Practice rules are classical association rules \cite{agrawal1993mining,singh2006market} and profit-aware item ranking by margin-weighted popularity \cite{louca2019profit,debiasio2024economic}; they measure the value of correct economics. A published recommender, TIFU-KNN, is described below. The strongest tier is the five heuristics of Table~\ref{tab:strategies}---including independent selection as in Singh et al.'s candidate lists and the separable regime of Ban et al. \cite{singh2006market,ban2024substitutes}---plus a post-hoc portfolio of all of them. They use the same fitted model, margins, and constraints as \SB{}, so they isolate the value of search alone: against them the exact policy nearly ties in ordinary markets, and its gains concentrate where offers pay only together (Section~\ref{sec:insights}). Complete Journey's source, Ta Feng, and Instacart are also the standard next-basket benchmarks \cite{li2023nbrcheck}, so we also run TIFU-KNN \cite{hu2020tifuknn}, a personalized-frequency method that matches or exceeds deep next-basket models on them. Because it uses the shopper's purchase history, which \SB{}'s anonymous model does not, the simulated shopper receives a personal habit: add-on purchase rates are scaled by $\pi^\omega$, where $\pi$ is the customer's own category purchase rate relative to the population. Used as the recommender, TIFU-KNN leaves 17--111\% of attainable offer margin unrealized; reranked by margin, it still leaves 11--52\% (Table~\ref{tab:priorart}). The losses arise where \SB{}'s accounting applies: on Complete Journey 44\% of the upgrades it shows displace margin, on Ta Feng short histories leave 55\% of premium alternatives unscored, and on Instacart it picks habitual add-ons rather than those the current basket makes likely. As one extra regressor in \SB{}'s response model, the same forecast adds 0.3--0.4\% on Complete Journey and Ta Feng and 6.0\% on Instacart when habits matter ($\omega=1$), and at most 0.012\% when they do not. Next-basket recommenders supply better forecasts; \SB{} converts them into margin.

\paragraph{Joint selection monetizes supported interactions.} EB shrinkage reduces the S1 exact-minus-independent gain from \$3.50 to \$0.91 per 1{,}000 baskets on Complete Journey (0.17\% of the oracle), eliminates the Ta Feng gain, and retains \$47 of \$53 on Instacart (10.2\%). Carbo-Loading supplies a complementary real-price result: exact gains \$1.65 (0.73\%) under the in-family truth S1, 1.22\% under the out-of-family truth S2, and 2.63\% when coupling is doubled. Its real-label gain survives both item- and line-label placebos, which yield \$0.13 and $-\$0.22$, and EB shrinkage leaves it unchanged because the training sample is large. A second measure uses a narrower base, \emph{coupling value}: the part of the oracle policy's margin that exists only because of cross-action interactions, i.e., its value minus the value of the same offers with coupling switched off. In Instacart and Carbo-Loading, joint selection captures 32--52\% of that coupling value (Appendix Table~\ref{tab:couplingshare}). The method converts a comparable share of interaction opportunity into margin; what differs most across markets is whether interactions carry 2\% or more than 30\% of attainable value.

\subsection{Decision-Aware Market-Basket Analysis}\label{sec:marketbasket}
StoBasket turns classical market-basket analysis from a descriptive association report into a decision-aware economic diagnostic. It keeps the familiar support, confidence, and lift \cite{agrawal1993mining,singh2006market} and adds premium-versus-regular anchors, empirical-Bayes shrinkage and placebos, signed cross-action effects, contribution margin, replacement-versus-addition semantics, and the operational test: whether an association is large enough to change the optimized offer set.

\paragraph{Scaling in common notation.} Consider $N$ baskets over a catalog of $d$ items, with $b_n$ items in basket $n$. Classical mining of associations up to order $q$ faces $\sum_{j\le q}\binom{d}{j}=O(d^q)$ candidate itemsets---polynomial in the catalog but exponential in the order---and counting requires $\sum_n\binom{b_n}{q}$ \emph{updates}, each adding one to the tally of a combination found in a basket; a 10-item basket alone contains 210 four-item combinations. \SB{} pools items into $G$ anchor groups and $M$ add-on cells ($G,M\le d$) and estimates one signed premium-versus-regular contrast per anchor/add-on pair plus selected local moments: $O(GM+|E|)$ quantities from $\sum_na_nm_n$ sparse updates, where $a_n,m_n\le b_n$ count the anchor groups and add-on cells in basket $n$. Neither count grows with $q$, because effects of several offers enter through the additive channel summary rather than as separate itemsets. At one million Instacart orders (46{,}798 products, 10.1 items per order, 134 aisle-level groups), four-item mining faces $2\times10^{17}$ candidate itemsets and $3.1\times10^9$ subset updates, whereas \SB{} estimates 17{,}956 contrasts from $7.6\times10^7$ updates, with the average anchor group observed in about 54{,}000 baskets (Appendix Table~\ref{tab:mbascaling}). Pair counting costs about the same in both methods; the gap opens with association order and with the number of quantities each basket must support. A single support threshold over that large candidate space creates the rare-item dilemma---high support discards infrequent but valuable relationships, low support yields a combinatorial number of noisy patterns \cite{liu1999multiple}---whereas \SB{} shrinks noisy contrasts toward zero by empirical Bayes. Because the same contrasts already feed the optimizer, the report is a natural byproduct rather than a second mining model. It is sparsity-aware, not sparsity-proof: relationships with no exposure remain unidentified, and a genuinely irreducible three- or four-item effect outside the additive summary requires an extension.

The patterns are data-derived. In Carbo-Loading's training receipts, baskets anchored by premium pancake mix contain Aunt Jemima syrup 11.3\% of the time, versus 3.3\% for regular pancake-mix baskets (a $3.43\times$ relative rate); private-label syrup appears 9.4\% versus 28.4\%. Premium pasta-sauce baskets contain Barilla pasta 8.2\% versus 4.8\% (Appendix Table~\ref{tab:marketbasket}). Like all market-basket statistics, these are observational; they give an interpretable opportunity report before any policy is served, and the implementation is suitable for separate open-source release as a market-basket-analysis component.

\subsection{Optimization, Learning, and Deployment Insights}\label{sec:insights}

\paragraph{Real coupling does not always require exact search.} The same Carbo-Loading report reveals coherent complement and avoidance patterns, but alternating search captures them too; the heuristic portfolio remains within 0.14\% of exact even when display cost is twice standalone cross-sell value. Online Retail II and the noise-corrected Complete Journey and Ta Feng results show the same operational conclusion for different reasons: either coupling is weak or local search captures it. These markets benefit from the \SB{} economic objective without paying for exact search.

\paragraph{Coupling plus display friction: exact becomes necessary.} Instacart's organic and premium tiers induce stronger upgrade-to-add-on associations. With separate budgets and display cost $\tau=1.5$ times mean standalone add-on value, the heuristic portfolio gives up 16\% of oracle margin. At $\tau=2$ it gives up 65\% under fixed shrinkage and 64\% under EB shrinkage; when coupling is also doubled, the fixed-shrinkage gap is 78\% (\$138 of \$177 per 1{,}000 baskets). Exact search remains within 0.2--0.5\% of the oracle in the default calibration. Here each component can be unattractive after display cost although the coupled pair is profitable, so every local method refuses the first move---the real-basket analogue of Proposition~\ref{prop:staging}(iii). The size of the failure is price-sensitive: at $\tau=2$, the portfolio gap is 89\% with a 10\% organic premium, 5.8\% with a 40\% premium, and 2.0\% after item-level price noise; exact regret is at most 0.5\% except in the low-value 10\% arm. Shared offer budgets break naive alternating by 1.3\%, 1.5\% and 21\% on Complete Journey, Ta Feng and Instacart, respectively, but enumerating every fixed upsell/add-on budget split repairs the gap to at most 0.07\%. Thus shared budgets alone do not establish a need for exact search; bundle-only profitability does, while prices determine how often that regime occurs. In this regime display costs consume most of the value---attainable margin falls from \$474 to \$46 per 1{,}000 Instacart baskets at $\tau=2$---so the 65\% heuristic loss is about 3 cents per basket. The finding concerns reliability: no heuristic signals when it has entered this regime, and prices alone move its loss from 2\% to 89\%. Appendix~\ref{app:batch2} reports the full map.

\paragraph{Model adequacy and economic robustness.} Under the stress truth S3 (Section~\ref{sec:empirical}), a base model that omits same-department cannibalization and offer fatigue loses 12--17\% of oracle margin. Adding those two mechanisms to the fitted model recovers their true parameters and reduces regret to 0.13--0.32\% even with the test-period shift. This controlled result shows that richer response structure can repair known misspecification, not that unknown misspecification is harmless. Bundle economics also matter: on Instacart, ignoring a rebate equal to the add-on margin loses \$34 per 1{,}000 baskets (7.6\% of oracle), whereas alternating search with the rebate-aware objective stays within \$1.3. Dominick's accounting data give national-brand and private-label gross margins of 22.1\% and 36.8\%; substituting these real rates for Complete Journey's assumed rates changes the joint gain only from 0.65\% to 0.58\%.

\begin{figure*}[t]
\centering\includegraphics[width=0.92\textwidth]{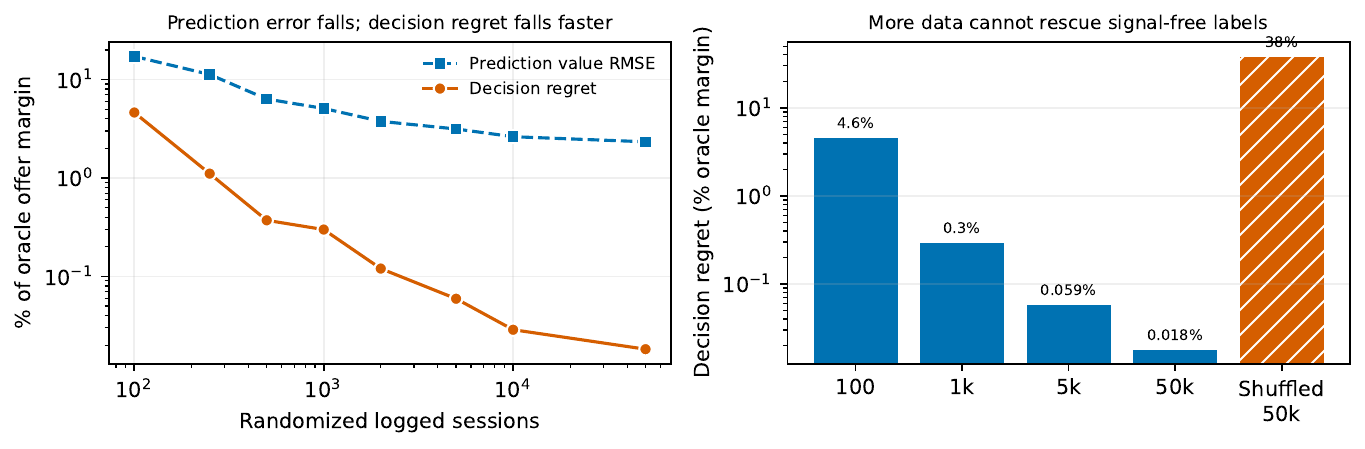}
\caption{\textbf{How much prediction quality matters} (Instacart, exact optimizer held fixed). Left: as randomized logged sessions grow, prediction error in policy values (squares, dashed) falls slowly while decision regret (circles, solid) falls much faster. Right: decision regret for learned predictors versus a predictor fit to 50{,}000 sessions with shuffled outcomes (hatched). Percentages are of oracle offer margin; both axes in the left panel and the right panel's vertical axis are logarithmic.}
\label{fig:prediction}
\end{figure*}

\paragraph{How much does prediction quality matter?} \SB{} accepts any coherent predictor, so we ask how much the response predictor, often called the propensity model, changes the economic result when everything else is fixed. On Instacart, the market with the strongest surviving coupling, truth, candidates, budgets ($L=K=2$), economics, and the exact optimizer are identical in every arm. Learned predictors fit the same three-channel model---a logistic regression for upsell acceptance, one correlation factor for joint acceptance, and a linear regression for add-on purchase (Appendix~\ref{app:empirical})---to between 100 and 50{,}000 randomized offer sessions. A signal-free predictor fits it to 50{,}000 sessions whose outcomes are permuted across all sessions and items: overall acceptance rates survive, but which items and offers work does not, so it effectively ranks offers by margin alone (Figure~\ref{fig:prediction}; Appendix Table~\ref{tab:prediction}). Prediction quality is decisive at the bad end: the signal-free predictor loses 38.4\% of attainable offer margin, and 100 valid sessions lose 4.6\%. Decision value then saturates quickly: 500 sessions come within 0.37\% of the oracle, 1{,}000 within 0.30\%, and 5{,}000 within 0.06\%. Economic regret falls far faster than prediction error: from 1{,}000 to 50{,}000 sessions, policy-value error falls only from 5.1\% to 2.3\% and the oracle's exact offer set is chosen 85\% and 96\% of the time. Competing offer sets are often nearly tied---the median gap between the best and second-best policy is \$0.003 per basket---so residual errors usually swap one near-optimal set for another, and the $2\varepsilon$ bound of Eq.~\eqref{eq:regret} is a worst case. Because the predictor shares the simulated truth's functional form, these curves are a best case; Section~\ref{sec:scope} brackets misspecification.

\paragraph{Profit saturates before accuracy: security selection and portfolio construction.} The closest practitioner analogy is active equity management, which separates \emph{alpha forecasting} from \emph{portfolio construction}: a model scores expected returns, security selection ranks the stocks, and an optimizer converts the scores into positions under risk and capacity constraints. Forecast skill is summarized by the \emph{information coefficient}, the correlation between scores and subsequent returns, because a useful ranking can build a good book even when return magnitudes are poorly calibrated; modest skill applied across many independent opportunities produces meaningful performance \cite{grinold1989fundamental,grinold2000active}. In \SB{}, predicted offer value is the alpha signal, ranking candidate offer sets is security selection, display and compatibility limits are portfolio constraints, and exact optimization is portfolio construction. Decision analysis calls the resulting insensitivity a \emph{flat maximum} \cite{vonwinterfeldt1982costs}, and predict-then-optimize formalizes that decision loss, not forecast error alone, is the relevant metric \cite{elmachtoub2022smart}.

\paragraph{Is calibrated coupling real?} Scaling channel effects by $\gamma$ gives zero gaps at $\gamma=0$ and increasing gaps thereafter (Figure~\ref{fig:empirical}). EB shrinkage drives line-placebo gains to approximately zero, retains mean coupling weight 0.89 on Instacart and 0.50 on Complete Journey, and detects no signal on Ta Feng. Carbo-Loading strengthens this diagnostic: its 2.5 million training baskets make shrinkage negligible, yet both label placebos remove the gain while cross-commodity effects remain under real labels. We therefore treat Instacart and Carbo-Loading as the supported coupling calibrations, with Carbo adding real paid prices. These tests establish that the receipt associations are structured rather than sampling noise; they do not turn observational co-purchase into a causal offer effect.

\paragraph{Randomized evidence for an exposure channel.} Building on exposure-aware recommenders that distinguish awareness from preference \cite{sato2020exposure}, we test whether neighboring display changes response in 454{,}265 screens from the Open Bandit Dataset's uniform-random fashion slates \cite{saito2020open}. Holding focal item and position fixed, each same-category neighbor at the coarsest seven-level category lowers focal click probability by 12.3\% (slate-randomization $p=0.012$, Holm-adjusted $p=0.036$ across three category resolutions). Neighbor attractiveness and the other category resolutions are null. Clicks are nine times more likely when a neighbor is clicked, but that association is not causal. This experiment supplies causal evidence for the exposure channel used by our economic policy. Its scope is fashion clicks; a field experiment will test the complete mechanism for grocery purchases.

\paragraph{Graph structure and exactness at scale.} Real upsell interaction graphs are seldom forests: half of Complete Journey's graphs over upgrade-eligible items contain a cycle, and 94\% do when eight or more items are eligible. With up to 20 upsell groups, enumeration and the tree-decomposition DP agree on all 8{,}040 checked baskets. Exact serving is fast once the enumeration is vectorized: scoring every feasible upsell policy at once with array operations, rather than one at a time in a Python loop, solves all 3{,}600 held-out baskets at $L=4$ in 0.13\,ms median, 30\,ms p99, and 111\,ms maximum on one core, with values matching brute force or DP on all 8{,}100 solves. An earlier scalar implementation took 1.1\,ms median and up to 3.0\,s (Appendix Table~\ref{tab:hybrid}). The slow tail came from Python per-policy overhead on dense 20-group baskets, not from the problem's difficulty, so no native-code rewrite is needed at these budgets. Budgets above $L=4$, where enumeration grows as $O((GJ)^L)$ and graphical DP becomes the scalable route, remain untested.

\begin{table*}[t]
\caption{\textbf{How StoBasket advances representative foundations.} Prior work supplies essential capabilities; StoBasket connects them in a joint post-basket decision over interacting replacement and addition offers.}
\centering\footnotesize
\renewcommand{\arraystretch}{1.2}
\begin{tabular}{p{0.19\textwidth}p{0.32\textwidth}p{0.41\textwidth}}
\toprule
Prior work & What it already does & What StoBasket adds\\\midrule
Profit-aware recommendation \cite{debiasio2024economic,louca2019profit,debiasio2024model} & Scores or trains rankings of individual products using provider profit, relevance, or a tradeoff between them. & Optimizes a \emph{set of edits} to a basket already in progress. An upsell replaces an incumbent and earns only the margin difference; a cross-sell adds an item and earns incremental margin; their values can interact.\\
Association rules and market-basket analysis \cite{agrawal1993mining,singh2006market} & Finds products that co-occur, reports support/confidence/lift, and may rank upsell or cross-sell candidates by association, acceptance, revenue, or margin. & Uses the association report as an interpretable diagnostic, then adds signed premium-to-complement effects, shrinkage/placebos, action type, contribution margin, and a feasible joint offer decision.\\
Upgraded substitutes and co-purchases \cite{ban2024substitutes} & Places substitute upgrades and co-purchases in one business problem; under the studied response assumptions, the component optimizations separate. & Makes the upsell decision change cross-sell response and value through displayed, accepted, and jointly accepted channels; proves when separation is safe and when it can fail arbitrarily.\\
Next-basket recommendation \cite{hu2020tifuknn,li2023nbrcheck} & Forecasts which items a shopper will buy next from purchase history, with strong reproducible results on the dunnhumby, Ta Feng, and Instacart benchmarks. & Turns that forecast into margin by using it as a response feature, then deciding which replacement and addition offers to serve. On the same benchmarks it adds up to 6\% as an input, whereas serving the forecast directly leaves 11--52\% unrealized even after margin reranking (Table~\ref{tab:priorart}).\\
Exposure modeling \cite{sato2020exposure} & Separates whether a user becomes aware of an item from whether the user prefers or purchases it, and models recommendation influence on exposure. & Places displayed-offer, accepted-offer, and joint-acceptance effects inside one expected-margin policy over replacements and additions.\\
Cross-category basket retrieval \cite{acar2026crosscategory} & Retrieves complementary products for a checkout-adjacent recommendation surface and debiases transactional complement evidence. & Takes retrieved candidates as inputs, values displaced and incremental margin, and jointly chooses which replacement and addition offers to serve.\\
Basket and choice models \cite{ruiz2020shopper,ke2022cross,elhousni2025markovian} & Learn substitution/complementarity or optimize assortments and linked category choices. & Starts from the shopper's observed basket and chooses typed edits whose acceptance probabilities and economic values depend on the whole displayed policy.\\
Structured optimization \cite{lo2019synergy,delpia2025factor} & Provides sparse-interaction and factorized binary-optimization results, including geometric and graphical techniques. & Combines those foundations for the StoBasket objective: score-arrangement enumeration outside and count-augmented tree-decomposition dynamic programming inside, with exactness for a stated response class.\\
\bottomrule
\end{tabular}\label{tab:prior}
\end{table*}
\section{Relationship to Prior Work}\label{sec:related}
Table~\ref{tab:prior} summarizes the foundations and the advance. The new object is the \emph{joint post-basket edit policy}, which differs from item ranking or assortment in three ways. Actions are typed: \texttt{REPLACE} removes an incumbent and earns a margin difference, whereas \texttt{ADD} earns incremental margin. Response is policy-conditioned across action types, through separate display, acceptance, and joint-acceptance channels. And the optimization results match that decision: when separate pipelines are exact, why staged methods have no constant-factor guarantee under coupling, and a global certificate for a structured class. The closest precedents each supply part of this object. Singh et al. rank upsell and cross-sell candidates for a current basket \cite{singh2006market}; Ban et al. jointly formulate upgraded substitutes and co-purchases under responses that separate \cite{ban2024substitutes}; Acar retrieves cross-category candidates \cite{acar2026crosscategory}; next-basket recommenders such as TIFU-KNN forecast purchases that \SB{} can use as an input \cite{hu2020tifuknn,li2023nbrcheck}; linked-category choice models represent related decisions without typed basket edits \cite{ke2022cross,elhousni2025markovian}; and Lo--Topaloglu and Del Pia supply the sparse-interaction and factorized-optimization techniques that the exact solver combines \cite{lo2019synergy,delpia2025factor}. Causal-mediation and exposure-bias methods underlie the logging requirements \cite{imai2010mediation,chen2020bias}.

\section{Discussion and Limitations}\label{sec:limitations}
\subsection{Six Key Insights}
Percentages below are shares of attainable incremental \emph{offer} margin, not total store profit.

\paragraph{1. Optimize margin, not recommendation scores.}
Across all five markets, treating an upsell as a replacement and a cross-sell as an addition changes which offers should be shown. Even after popularity is weighted by margin, the remaining opportunity is \textbf{1.0--11.4\%}. A published next-basket recommender used directly leaves \textbf{11--52\%} unrealized after margin reranking, but adds up to \textbf{6\%} when its forecast is fed into \SB{}. The product implication is simple: a better predictor is an input; it is not a substitute for the economic decision layer.

\paragraph{2. Historical data reflects the offer program that produced it.}
Joint selection gains \textbf{10.2\%} on digitally active Instacart, compared with \textbf{0.17\%} on Complete Journey and approximately zero on Ta Feng. Weak measured coupling can describe current merchandising rather than the ceiling of customer response. A retailer should therefore test joint offers before concluding that its market has no interaction opportunity.

\paragraph{3. The response layer is also market-basket analysis.}
The same estimates used to choose offers produce premium-conditioned support, confidence, lift, signed complement or avoidance effects, and contribution margin. For example, premium pancake-mix baskets contain Aunt Jemima syrup \textbf{11.3\%} of the time versus \textbf{3.3\%} for regular-mix baskets, a \textbf{$3.43\times$} relative rate. This decision-aware market-basket report is an interpretable product byproduct, while remaining explicitly observational rather than causal.

\paragraph{4. Ranking offer sets matters more than perfect probability forecasts.}
With 1{,}000 randomized sessions, policy-value error remains \textbf{5.1\%}, yet the resulting margin regret is only \textbf{0.30\%}; residual errors usually exchange nearly tied policies. In contrast, a predictor trained on shuffled outcomes loses \textbf{38\%}. The practical target is therefore useful signal and correct decision ranking, not marginal accuracy improvements after economic value has saturated.

\paragraph{5. Exact optimization provides a trust anchor.}
For every instance satisfying Proposition~\ref{prop:mainexact}, \SB{} returns the globally best feasible offer set under the supplied model: no heuristic can produce greater predicted margin for that same basket, objective, and response estimates. Retailers may still deploy faster heuristics; the exact solution supplies the benchmark that certifies them and measures their optimization gap. That distinction matters little in ordinary markets, where strong heuristics stay within \textbf{0.09--0.14\%}, but matters sharply in the bundle-only regime, where they lose \textbf{65\%}. A heuristic gives no warning when it enters that regime---Instacart price assumptions alone move its loss from 2\% to 89\%---so the value of the guarantee is that a retailer need not guess. Exact serving takes \textbf{0.13\,ms at the median} and at most \textbf{111\,ms} on every tested real basket. The guarantee removes search error; it does not correct bad response data or model misspecification.

\paragraph{6. Structure makes that certificate computationally practical.}
The unrestricted interacting-offer problem is NP-hard (Appendix~\ref{app:hardness}). Fixed-dimensional response channels and local interaction factors convert the stated class into an exact polynomial-time problem, and two independent exact methods agree on all \textbf{8{,}040 of 8{,}040} checked real baskets. The contribution is therefore not merely a smarter heuristic: it identifies the conditions under which a global certificate is both mathematically valid and operationally affordable.

\subsection{Industrial Impact and Deployment}\label{sec:impact}
The gains are market-dependent, and that dependence is itself actionable. Joint selection earns the most where premium tiers and add-ons are strongly linked, as in Instacart; where measured coupling is weak, the economic objective still recovers margin and fast heuristics suffice. We interpret this pattern as a statement about merchandising practice as much as about shoppers. Receipts record behavior under a retailer's existing offer regime. Instacart operates a digital platform on which ranking models were already in production when its public data were released \cite{stanley2017instacart}; today it serves in-cart complement suggestions, real-time personalized recommendations, and brand coupons that appear on item cards and prompt shoppers in the cart \cite{jia2023onemodel,ruan2023dynamic,instacart2026promotions}. Its baskets display the strongest premium-to-add-on coupling in our study. Complete Journey and Ta Feng record more conventional retail settings, where premium tiers and complements are less often presented together, and their receipts show weak coupling. A retailer that rarely pairs premium items with complements will observe little such association in its data; that silence measures the current regime, not the ceiling of the opportunity.

This interpretation is a hypothesis that our observational data cannot settle, but it carries a clear prescription. A retailer whose historical data shows little joint-selection value should not conclude that its shoppers lack the opportunity. It should run randomized joint offers. Randomization refreshes the data with logged exposure and acceptance, identifies the display and acceptance channels that receipts cannot separate, and reveals whether coupling emerges once offers are actually made. The relationship is self-reinforcing: the more actively a retailer experiments with joint upsell and cross-sell offers, the richer the interaction data becomes, and the more value \SB{} can extract from it.
\subsection{Scope and Next Steps}\label{sec:scope}
\paragraph{(1) Scope of the exactness guarantee.} The guarantees apply to coherent acceptance moments, factor-local dependence, bounded interaction width, and fixed response dimension. Arbitrary choice laws, optimized promotion amounts, and high-order bundle conditions may require different algorithms. Latency was measured only for $L\le4$ and at most 20 upsell groups; larger offer budgets are where graphical DP should matter most.

\paragraph{(2) Cost of the structured approximation.} Exactness holds within the structured projection of Eq.~\eqref{eq:structured}, not for arbitrary customer response. If the projected value approximates the true value within $\varepsilon$ for every offer set, the plug-in argument of Eq.~\eqref{eq:regret} bounds structural regret by $2\varepsilon$; without such a bound, a nonlinear response law can reverse policy rankings. Empirically, under the out-of-family truth S2 (Section~\ref{sec:empirical}), whose logistic saturation, latent classes, and three-way response lie outside the fitted family, regret is only 0.02--0.69\% across five datasets. Under the stress truth S3, omitting cannibalization, fatigue, and temporal shift costs 10.6--19.3\%, and modeling cannibalization and fatigue reduces it to 0.08--0.32\%. Smooth curvature can be benign; omitted mechanisms that create new interactions are not.

\paragraph{(3) Evidence is calibrated simulation, not live uplift.} Receipt co-occurrence alone cannot identify the causal effect of displaying recommendations \cite{imai2010mediation,chen2020bias}. Carbo-Loading has real paid prices but assumed margins, and alternative imputations of Instacart's missing prices move the high-friction heuristic gap at $\tau=2$ from 2\% to 89\%. The randomized fashion result independently identifies a display effect, but for clicks rather than grocery purchases.

\paragraph{(4) Learning results assume the right model family.} The stress-test recovery and the prediction-quality results both use one calibrated truth and a correctly specified response family, so a field test should confirm the required log size in each market. Without a simulator, a retailer can check the same property on its own data: hold out randomized sessions and compare predicted with realized offer margin, retrain on growing log sizes to see whether policy value has plateaued, and estimate new policies off-policy from logged selection probabilities before serving them (Appendix~\ref{app:abtest}).

\paragraph{(5) Next step: an A/B test any retailer can run.} A retailer can measure \SB{}'s uplift on its own traffic in five steps (Appendix~\ref{app:abtest}). First, use the market-basket diagnostic of Section~\ref{sec:marketbasket} to choose a category with strong premium-to-add-on coupling. Second, randomize each basket, by a hash of its session identifier, to the current recommender or to \SB{}, keeping a small no-offer arm and a 5--10\% uniformly random arm for re-estimation. Third, log what was shown, its selection probability, what was accepted, the final basket, and realized prices, costs, and rebates. Fourth, compare arms on offer-attributable incremental contribution margin per basket. Fifth, size the test from Appendix Table~\ref{tab:power}: about 1{,}400 baskets per arm under Instacart-like coupling with offer-level logs, but 20--50 times more if only till receipts are available.

\paragraph{(6) Next step: campaigns that make coupling learnable.} Weak coupling in historical receipts can reflect how offers were run rather than how shoppers respond. When upsell and cross-sell programs are planned, owned, and measured separately, premium items and their complements are seldom shown together; the logs then contain little joint exposure, and the coupling coefficients are not identified. This is the positivity requirement of causal and off-policy evaluation \cite{chen2020bias,saito2020open} applied to merchandising: an effect that is never exposed cannot be learned, and its absence from the data is not evidence that it is zero. The same separation also produces the staging trap of Proposition~\ref{prop:staging}, because each program judges its offers alone. The remedy is operational rather than algorithmic: co-display premium anchors with candidate complements under one contribution-margin objective, keep a small randomized share with logged selection probabilities, record shown and accepted offers separately, and direct exploration toward the cells that the shrinkage diagnostic marks as uninformative. Appendix~\ref{app:campaign} gives the full recipe.

\section{Conclusion}
\SB{} contributes a formulation before it contributes a solver. To our knowledge, it is the first formulation of joint upsell and cross-sell selection in which the two responses are coupled; separable treatments are its zero-coupling special case. With a correctly structured response model fit to a few thousand or more randomized offer sessions, \SB{} comes within 0.5\% of the best attainable offer margin in every default-price setting of five retail markets, and within 0.7\% when the true response lies outside its model family; where it falls short, the cause is the response model, never the search. Across those markets, correctly accounting for displaced and incremental margin is the largest and most consistent source of value. Joint selection pays most where interactions are strong, beating independent selection on Instacart by 10.2\% of attainable offer margin. Exact optimization is a targeted final layer: unnecessary when alternating search captures the interaction, as on Carbo-Loading, but decisive when actions pay only together. Its certificate is practical---exact methods agree on every checked real basket, and vectorized exact search takes 0.13\,ms at the median and at most 111\,ms---and serves as a trust anchor that removes search error, leaving prediction quality and economic inputs as separate, measurable risks. The predictor must rank offer sets correctly rather than forecast perfectly, so a few thousand randomized sessions suffice. The next step is a randomized field test in an interaction-heavy category, with shown offers, acceptances, realized costs, and contribution margin logged end to end.

\bibliographystyle{ACM-Reference-Format}
\bibliography{references}
\appendix
\input{appendix}
\end{document}

%% file: tables/empirical_economic_summary.tex
\begin{tabular}{lrrr}
\toprule
Market & Economic-policy gain & Joint-selection gain & Exact-search premium\\
 & vs.\ margin-pop.\ (\%) & vs.\ independent (\%) & vs.\ portfolio (\%)\\
\midrule
Complete Journey & 2.4 & 0.17 & $\leq 0.09$\\
Ta Feng & 1.5 & $\approx 0$ & $\leq 0.09$\\
Instacart & 11.4 & 10.2 & 63.8$^\dagger$\\
Online Retail II & 1.0 & 0.34 & $\approx 0$\\
Carbo-Loading & 4.7 & 0.73 & $\leq 0.14$\\
\bottomrule
\end{tabular}

%% file: tables/empirical_e1.tex
\begin{tabular}{llrrrrrrrrr}
\toprule
Data & Sim. & Oracle & Exact & Indep. & Upg.-first & Add-first & Greedy & Altern. & Margin-pop. & Assoc.\\
\midrule
Complete Journey & S1 & 536 & 536 & +3.5 & +1.1 & +2.6 & +0.4 & +0.0$^\dagger$ & +12.6 & +534.2\\
 & S2 & 509 & 507 & +5.8 & +1.3 & +5.8 & +2.1 & +0.1 & +13.2 & +493.6\\
 & S3 & 373 & 313 & +2.4 & +0.5 & +1.7 & +0.2 & +0.0$^\dagger$ & +7.9 & +308.3\\
Ta Feng & S1 & 20576 & 20570 & +71.3 & +25.8 & +46.2 & +10.0 & +0.3 & +300.7 & +10746.4\\
 & S2 & 19523 & 19498 & +104.0 & +52.9 & +56.7 & +21.7 & +1.0 & +296.2 & +10651.8\\
 & S3 & 14346 & 12069 & +46.0 & +10.0 & +35.5 & +4.7 & -0.2$^-$ & +170.2 & +6371.9\\
Instacart & S1 & 474 & 474 & +53.4 & +17.0 & +35.0 & +13.3 & +0.4 & +54.1 & +89.1\\
 & S2 & 546 & 542 & +110.5 & +34.9 & +79.1 & +42.8 & +1.4 & +106.3 & +143.9\\
 & S3 & 303 & 266 & +10.6 & +8.9 & -0.7$^\dagger$ & +4.1 & +0.2$^\dagger$ & +12.0 & +30.9\\
\bottomrule
\end{tabular}

%% file: tables/empirical_prior_art.tex
\begin{tabular}{llrrrrrr}
\toprule
 & & \multicolumn{2}{c}{StoBasket} & \multicolumn{2}{c}{TIFU-KNN} & Margin- & Assoc.\\
\cmidrule(lr){3-4}\cmidrule(lr){5-6}
Dataset & $\omega$ & + TIFU & anonymous & $\times$ margin & relevance & popularity & rules\\
\midrule
Complete Journey & 0 & \textbf{99.9} & \textbf{99.9} & 88.7 & $-$10.8 & 97.6 & $-$2.3\\
 & 0.5 & \textbf{99.7} & 99.6 & 88.8 & $-$10.1 & 97.3 & $-$1.7\\
 & 1 & \textbf{99.3} & 98.9 & 88.4 & $-$9.2 & 96.7 & $-$0.4\\
\addlinespace
Ta Feng & 0 & \textbf{100.0} & \textbf{100.0} & 48.2 & 37.0 & 98.5 & 47.3\\
 & 0.5 & \textbf{99.9} & 99.8 & 49.3 & 37.9 & 98.4 & 47.7\\
 & 1 & \textbf{99.8} & 99.5 & 50.7 & 39.4 & 98.1 & 48.5\\
\addlinespace
Instacart & 0 & \textbf{100.0} & \textbf{100.0} & 74.4 & 75.5 & 88.6 & 81.1\\
 & 0.5 & \textbf{99.2} & 97.8 & 80.7 & 81.2 & 87.7 & 81.2\\
 & 1 & \textbf{97.9} & 91.9 & 83.7 & 83.2 & 83.2 & 78.6\\
\bottomrule
\end{tabular}

%% file: appendix.tex
% StoBasket technical appendix: response model, proofs of the main-text propositions,
% synthetic verification, and real-data-calibrated evaluation.

\section{Response model and economic scope}\label{app:model}

\subsection{Feasible basket edits}
Partition a current basket into $G$ incumbent groups. Each group $g$ permits one chosen upgrade offer $x_g\in\{0,1,\ldots,J_g\}$, with $x_g=0$ denoting no offer. At most $L$ groups receive an upgrade offer. Independently, choose a set of offered additions $C\subseteq [M]$ with $|C|\le K$. These are \emph{simultaneously selected offers}, not assumed purchases. An accepted upgrade removes its incumbent and inserts an alternative. A purchased add-on inserts an item without removing an incumbent. The action constraints prevent conflicting replacements by construction.

Let $Y_g\in\{0,1\}$ indicate acceptance of the offered upgrade (zero when $x_g=0$). For fixed context $s$, write $p_g(x_g;s)=\Pr(Y_g=1\mid x_g,s)$ and, for edges $(g,h)$ of a graph $H=(V,E)$,
\begin{equation}\rho_{gh}(x_g,x_h;s)=\Pr(Y_g=Y_h=1\mid x_g,x_h,s).\label{eq:rho}\end{equation}
Crucially, $\rho_{gh}$ need not equal $p_gp_h$. We assume these first and edgewise second moments are \emph{policy-local}: changing offers outside their endpoints does not change these moments. We also assume the supplied moments come from a consistent joint distribution. On a forest, compatible edgewise binary marginals specify at least one such distribution.

\subsection{Offered and realized utility effects}
For each active group $g$, let $d_g(x_g)$ be the incremental accepted-upgrade margin, $t_g(x_g)$ its accepted latent displacement, $v_g(x_g)$ its \emph{presentation} feature, and $c_g(x_g)$ the cost of displaying it. Each offered complement $k$ has incremental margin $m_k\ge 0$ in this nonpromotional or fixed-terms subclass and display cost $c_k$. We allow pairwise upgrade effects on edges: accepted upgrades $g,h$ may contribute extra margin $\gamma_{gh}(x_g,x_h)$ and a latent interaction feature $\zeta_{gh}(x_g,x_h)$.

A structured \emph{conditional} add-on acceptance model is
\begin{equation}
\begin{split}
q_k(Y,x)={}&a_k+b_k^\top\sum_gY_gt_g(x_g)\\
&+e_k^\top\sum_gv_g(x_g)+h_k^\top\sum_{(g,h)\in E}Y_gY_h\zeta_{gh}(x_g,x_h),
\end{split}\label{eq:choice}
\end{equation}
where each feature may be vector-valued. Parameters must satisfy $0\le q_k(Y,x)\le1$ for every feasible action/acceptance realization. The model allows an offer to affect another item's purchase probability \emph{even when rejected}. We do not require add-on purchases to be conditionally independent, but assume their expected margin is additive and the purchase probabilities do not depend on which other additions are offered. These conditions specify the tractable conditional-response family analyzed below.

Define the finite-dimensional moment summary
\begin{equation}
\begin{aligned}
z(x)=\Big(&\sum_gp_g(x_g)t_g(x_g),\,\sum_gv_g(x_g),\\
&\sum_{(g,h)\in E}\rho_{gh}(x_g,x_h)\zeta_{gh}(x_g,x_h)\Big)\in\mathbb R^r.
\end{aligned}\label{eq:state}
\end{equation}
Absorb the appropriate coefficients into $\beta_k=m_k(b_k,e_k,h_k)$ and let $\alpha_k=m_ka_k-c_k$, where $c_k$ is the display cost of add-on $k$. The expected incremental margin of the policy is
\begin{equation}
\boxed{\begin{aligned}
F(x,C)={}&\underbrace{\sum_g[p_gd_g-c_g]}_{U(x)}
+\sum_{(g,h)\in E}\rho_{gh}\gamma_{gh}\\
&+\sum_{k\in C}\underbrace{(\alpha_k+\beta_k^\top z(x))}_{\ell_k(z(x))}.
\end{aligned}}\label{eq:objective}
\end{equation}
The optimization is $\max_{x,C}F(x,C)$ under $\sum_g\1\{x_g\ne0\}\le L$ and $|C|\le K$. All $p,\rho,d,t,v,\gamma,\zeta$ are evaluated at the chosen local actions and fixed observed context $s$; we suppress $s$ from most expressions for readability. The baseline incumbent margin cancels when taking incremental profit.

\subsection{Context conditioning}\label{sec:context}
Every quantity in this appendix is conditional on one fixed transaction context $s=(B,\ell,\tau,P,I,\ldots)$: the basket, store or verified region $\ell$, store-local time $\tau$, prices $P$, and available inventory $I$. Acceptance moments and add-on response coefficients may vary with $s$, and sparse location--time strata back off to coarser estimates. The exactness conditions below must hold for each fixed context; dependence induced by an arbitrary trained conditional model does not satisfy them automatically.

\subsection{Optional promotions: economic scope and tractability}\label{sec:promotion}
A recommendation selects what to suggest; a \emph{promotion} sets optional transaction terms such as a percentage discount, a coupon, or a bundle rebate. These are distinct decisions. Let $D$ be a set of fixed promotion terms, with $D=\varnothing$ denoting ordinary prices. In the unrestricted economic formulation, the response law is $P(Y\mid s,A,D)$ and net incremental contribution margin is
\begin{equation}\label{eq:promotion_general}
F(A,D)=\mathbb E\!\left[M_D(B\oplus Y)-M_0(B)\mid s,A,D\right]-c(A,D).
\end{equation}
Here $M_D$ includes net realized selling prices, product costs and any redeemed promotion funding, whereas $c$ comprises costs not already counted in $M_D$. Thus a discount is charged \emph{once}, only when its eligibility/redemption conditions hold. The baseline $M_0(B)$ refers to keeping the original basket at baseline transaction terms. Setting $D=\varnothing$ recovers the nonpromotional model.

A retailer-funded discount $\delta$ on accepted upgrade $g$ changes its accepted margin from $d_g$ to $d_g-\delta$ and may change $p_g$ to $p_g(D)$; a cross-sell discount similarly changes $m_k$ to $m_k-\delta_k$. These modified parameters enter Eq.~\eqref{eq:objective}, and the exact results apply, \emph{provided} its affine-response, local-dependence, and additive-margin assumptions continue to hold. A bundle rebate $r$ paid only if an upsell and a cross-sell are both accepted contributes $-r\Pr(Y_g=1,Z_k=1\mid s,A,D)$, a joint moment that must be modeled rather than recovered from marginals. Arbitrary bundle thresholds, price-dependent nonlinear response, and the joint choice of discount amounts can create new interactions, so exact solvability must then be rechecked.

\section{Structural characterizations}
\begin{proposition}[Moment sufficiency]\label{prop:moments}
Under Eq.~\eqref{eq:choice} and additive expected margins, $F(x,C)$ depends on the joint distribution of upgrade acceptances only through the marginals $p_g$ and selected joint moments $\rho_{gh}$ appearing in Eq.~\eqref{eq:state}. No higher-order moments or independence assumption are required.
\end{proposition}
\noindent\emph{Proof sketch.} Apply linearity of expectation to $Y_g$, $Y_gY_h$, and each conditional add-on probability. All terms then reduce to $p_g$ and $\rho_{gh}$, giving Eq.~\eqref{eq:objective}. See Appendix~\ref{app:proofs}.

\begin{proposition}[Exact separability condition]\label{prop:separability}
Assume no-offer is feasible, $z(0)=0$, $L,K\ge1$, and the feasibility constraints on upgrades and additions are separate. Then the objective can be written as $F(x,C)=U_0(x)+V_0(C)$ for all feasible pairs if and only if
\begin{equation}\beta_k^\top z(x)=0\quad\text{for every feasible $x$ and item $k$.}\label{eq:separation}\end{equation}
\end{proposition}
\noindent\emph{Proof sketch.} The mixed finite difference
\begin{equation}
D(x,k)=F(x,\{k\})-F(x,\varnothing)-F(0,\{k\})+F(0,\varnothing)
=\beta_k^\top z(x)\label{eq:diff}
\end{equation}
vanishes for every $x,k$ exactly when the cross term can be removed. A nonzero difference certifies objective-level coupling, but does \emph{not} imply that separate optimization necessarily selects a different optimum.

\section{An exact algorithm for structured interactions}
\begin{proposition}[Fixed-dimensional, bounded-treewidth exact solvability]\label{prop:poly}
Suppose (i) add-on response is affine as in Eq.~\eqref{eq:choice}, with additive expected add-on margins; (ii) every upgrade-side term---each expected margin in $U$ and each contribution to the summary $z$---depends only on the actions $x_{S_f}$ of a factor scope $S_f$, and the primal graph joining upgrade groups that share a scope has fixed treewidth $w$; (iii) only separate cardinality budgets $L$ and $K$ apply; (iv) all input coefficients are rational; and (v) the summary dimension $r$ is fixed. Then \SB{} optimization admits an exact algorithm running in time polynomial in the number of upgrade choices, the number of add-ons, and the coefficient encoding length.
\end{proposition}
Condition (ii) covers the single-group and edge moments of Eq.~\eqref{eq:objective}, as well as local exposure effects in which displaying a neighboring upsell changes $p_g$. It does not cover a graphical choice model whose induced marginals depend on the whole policy; such a model must first be shown to have locally factorized expected values.

\noindent\emph{Proof sketch.} For a fixed $x$, the best addition set consists of at most $K$ positive values of $\ell_k(z(x))=\alpha_k+\beta_k^\top z(x)$. The order and positivity of these affine scores can change only across hyperplanes $\ell_k=\ell_j$ or $\ell_k=0$. There are $O(M^2)$ hyperplanes in fixed dimension $r$, yielding $O(M^{2r})$ open cells. Enumerate one top-$K$ selection per cell, with tie handling. Some selection from these cells is optimal at the state $z(x^*)$ of a global optimum. For each such fixed $C$, the objective becomes a sum of local factors over the scopes $S_f$. A tree-decomposition dynamic program, with an additional state counting upgrade offers up to $L$, solves this exactly for fixed $w$. Maximizing over the enumerated $C$ recovers a global optimum. Full details are in Appendix~\ref{app:proofs}.

\noindent\textbf{Implementation.} The reference solver (\path{engine/treewidth_exact.py}) covers any fixed $r$ and arbitrary factor scopes. The outer step (\path{engine/arrangement.py}) returns, in exact rational arithmetic, a point in every full-dimensional cell of the score arrangement by adding the coordinate hyperplanes, slicing between consecutive vertex coordinates, and recursing on dimension; on a tie face, an adjacent cell's top-$K$ set is optimal, so cells suffice. This simple sampler is polynomial for fixed $r$ but targets verification rather than speed. The inner step (\path{engine/treewidth_dp.py}) is max-sum variable elimination along a min-fill ordering, with messages indexed by separator assignment and upgrade-offer count; the ordering affects only the cost, $O(G\,J_{\max}^{w'+1}(L+1)^2)$ per candidate set for induced width $w'$. The algorithm combines known arrangement and graphical-DP techniques rather than proposing them as new \cite{lo2019synergy,delpia2025factor}.

\section{Why structure is needed}\label{app:hardness}
Fixed response dimension alone does not make the problem tractable. The two reductions below show that the add-on and response restrictions of Proposition~\ref{prop:poly} carry the result.

\textbf{Unrestricted add-on interactions.} Allow pairwise add-on payoffs $\eta_{k\ell}$ as in Eq.~\eqref{eq:structured}. Given a graph $H$ and an integer $K\ge2$, create one add-on per vertex, no upgrade groups, zero unary values, add-on budget $K$, and $\eta_{k\ell}=1$ exactly for the edges of $H$. The objective counts edges among the chosen add-ons, so value $\binom K2$ is attainable exactly when $H$ has a $K$-clique. The optimization problem is therefore NP-hard even when the summary is identically zero.

\textbf{Nonlinear response.} Given positive integers $w_1,\ldots,w_G$ with sum $S$ and a target $T\in[0,S]$, let $L=G$ and $K=1$, let each upgrade be accepted with certainty when offered and change no margin, and let one add-on with unit margin and no display cost have purchase probability $q(x)=1-\big((\sum_gw_gx_g-T)/S\big)^2\in[0,1]$. Expected value $1$ is attainable exactly when some subset of the $w_g$ sums to $T$. A one-dimensional summary with a nonlinear response is thus already NP-hard; the affine response of Eq.~\eqref{eq:choice} is what restores tractability.

\section{Synthetic verification}
\subsection{Controlled validity tests}
All data below are synthetic. We generate multiple incumbent groups, several upgrade alternatives per group, a random forest of potential cross-upgrade effects, and complementary items. Coefficients are restricted or explicitly checked to ensure conditional add-on probabilities are within $[0,1]$.

\begin{table}[t]
\caption{Reproducible finite-instance verification; exactness means agreement with exhaustive search on the enumerated instances, not a proof of all-instance correctness.}
\centering\small
\begin{tabular}{lrr}
\toprule
Test & Instances & Failures\\
\midrule
Forest plus pair moments, $r=2$ &320&0\\
Nonempty interaction graph among those &289&--\\
Moment expansion vs. tree-joint law &120&0\\
Tree-decomposition DP vs.\ brute force &300&0\\
Treewidth solver, rational factors, $r\le2$ &240&0\\
\quad of which cyclic upgrade graphs &216&--\\
Treewidth solver, latent-class exposure model &120&0\\
\quad moment objective vs.\ full joint law &360&0\\
Generic DP equals forest and chain DPs &120&0\\
Cell sampler coverage, $r=1$--$4$ &56&0\\
Cell sampler equals face sampler, $r\le2$ &60&0\\
Treewidth solver, $r=3$, rational factors &48&0\\
Treewidth solver, $r=3$, latent-class model &36&0\\
Treewidth solver, $r=4$, rational factors &12&0\\
Staging traps, $n=1$--$6$ (all policies enumerated) &6&0\\
\bottomrule
\end{tabular}
\label{tab:checks}
\end{table}

The optimizer matched brute force on all 320 small forest instances, 289 with a nonempty upgrade-interaction graph, and the closed-form expected margin matched explicit summation over a correlated joint law to within $8.9\times10^{-16}$. These checks establish numerical consistency under synthetic parameters, not predictive accuracy.

\subsection{Bounded-treewidth upgrade graphs}\label{sec:twexp}
These checks test Proposition~\ref{prop:poly} beyond forests. The tree-decomposition program matched brute force on 300 random factor graphs (cycles, ladders, and partial 2- and 3-trees). The complete solver matched exhaustive search on 240 rational instances with $r\le2$ (216 cyclic) and on 120 coherent latent-class instances with exposure spillovers, in which displaying a neighboring upgrade shifts acceptance (109 cyclic, induced width up to four); there the moment objective matched the full joint law to within $1.3\times10^{-15}$. The cell sampler hit every sign pattern found by dense random sampling on 56 arrangements in dimensions one to four. In the model's natural dimension $r=3$, with separate accepted, displayed, and joint-acceptance coefficients, the solver matched exhaustive search on 48 rational and 36 latent-class instances, and on 12 instances with $r=4$. Larger $r$ would add response features rather than a new mechanism; the arrangement's $M^{O(r)}$ cells make bounded $r$ a substantive assumption. The 60-instance interaction sweep of Figure~\ref{fig:sweep} uses four upgrade groups, six add-ons, and budgets $L=K=3$; all baselines coincide with the optimum at zero coupling.

\subsection{Proof of the staging-gap proposition}\label{app:staging}
Write $x^0=\arg\max_xF(x,\varnothing)$ and $C^0=\arg\max_CF(0,C)$. The independent policy is $(x^0,C^0)$. Upgrade-first returns $(x^0,\arg\max_CF(x^0,C))$, whose value is at least $F(x^0,C^0)$ because $C^0$ is feasible. Add-on-first returns $(\arg\max_xF(x,C^0),C^0)$, likewise at least $F(x^0,C^0)$. Alternating best response starts from the better of the two and never decreases the objective, so it dominates both. Every policy is feasible, so none exceeds the optimum. This proves (i).

For (ii), Proposition~\ref{prop:separability} gives $F(x,C)=U(x)+\sum_{k\in C}\alpha_k$, with $U(x)=F(x,\varnothing)$. Because the budgets are separate, $\max_{x,C}F=\max_xU(x)+\max_C\sum_{k\in C}\alpha_k$, and $(x^0,C^0)$ attains both maxima. By (i) the sequential and alternating policies are also optimal. Greedy can fail even here. If $L\ge2$ and two upgrades each have value $-\epsilon$ alone and a pair margin $3\epsilon$ together, greedy never adds the first one. The same holds for a pair of add-ons with a positive interaction.

For (iii), take one upgrade with acceptance $p=1/2$, accepted margin change $-2\lambda$ and latent feature $0.6$, and one add-on with margin $10\lambda$, base acceptance $0.02$, slope $1.3$ on the accepted-upgrade feature, and display cost $0.5\lambda$, with $L=K=1$. Then $F(\varnothing)=0$, $F(\text{upgrade})=-\lambda$, $F(\text{add-on})=-0.3\lambda$ and $F(\text{both})=2.6\lambda$. Here $x^0$ and $C^0$ are both empty. The best addition given no upgrade is none, and the best upgrade given no addition is none. Alternating best response is therefore at a fixed point, and both greedy moves from the empty policy have negative gain. All five policies return $0$, while the optimum is $2.6\lambda>0$. No positive multiplicative guarantee $\mathrm{ALG}\ge c\cdot\mathrm{OPT}$ can hold, and the additive gap $2.6\lambda$ is unbounded.

For (iv), take $n$ groups, each with one upgrade of acceptance $p=1/2$, accepted margin change $-0.2$, no display cost, and accepted feature $t=1/n$. Take one add-on with margin $10$, base acceptance $0$, unit slope on the accepted-upgrade feature, and display cost $A=(n-\tfrac12)b$, where $b=5/n-0.1>0$ for $n<50$. Its conditional acceptance $\sum_gY_g/n$ lies in $[0,1]$ for every realization, so the response law is valid. An upgrade alone is worth $-0.1$. With the add-on shown, each offered upgrade adds $10\cdot\tfrac12\cdot\tfrac1n-0.1=b$. A policy with the add-on and $j$ upgrades is worth $jb-A$, which is negative for $j\le n-1$ and equals $b/2>0$ for $j=n$. Policies without the add-on are worth $-0.1j$. Every nonempty policy with at most $n$ offers is therefore negative. Any search that moves from the empty policy only to a policy within $n$ offer changes, and only on strict improvement, stops at $0$; this includes all five baselines. Exhaustive enumeration confirms the construction for $n=1,\ldots,6$. Scaling all margins and costs by $\lambda$ again makes the gap arbitrary. The traps depend on the starting policy: best-response iteration started from a policy that already shows the add-on reaches the optimum here. Whether a heuristic succeeds thus depends on such design choices, whereas the exact solver certifies optimality independently of them.

\section{Real-data-calibrated evaluation: data and design}\label{app:empirical}
\textbf{Claim boundary.} No public dataset records which upsells and cross-sells were displayed and whether they were accepted. All results in Section~\ref{sec:empirical} are therefore \emph{real-data-calibrated simulations}. Real baskets, prices, contexts and (for Ta Feng) costs define the decision problems; purchase patterns in training receipts calibrate simulated offer responses; policies are learned from simulated randomized offer logs and scored by their exact expected value under the simulator. The settings below were fixed in a written plan before the first full run.

\begin{table}[t]
\caption{Datasets. Splits are chronological; customer identifiers are not used.}
\centering\footnotesize
\begin{tabular}{p{0.2\columnwidth}p{0.22\columnwidth}p{0.22\columnwidth}p{0.22\columnwidth}}
\toprule
& Complete Journey & Ta Feng & Instacart\\\midrule
Source & 84.51\textdegree{} via \texttt{completejourney}, CC0 & Taiwanese supermarket, 2000--01 & 2017 public release \\
Train / test baskets & 102,604 / 37,434 & 88,527 / 31,051 & 400,000 (sample) / 131,209 \\
Split & weeks 1--39 / 40--53 & Nov--Jan / Feb & prior / final orders \\
Margin & assumed rates, 3 scenarios & \textbf{real cost} & imputed prices \\
Context & store, weekday, hour & residence area, weekday & weekday, hour \\
Premium tier & national vs.\ private label & price above subclass median & organic \\
\bottomrule
\end{tabular}\label{tab:data}
\end{table}

\textbf{Data sources and licenses.} All datasets are public and are used here for research. Complete Journey 2.0 originates from 84.51\textdegree{} and is used through the \texttt{completejourney} R package, released under CC0 \cite{boehmke2025completejourney}. Online Retail II \cite{chen2012onlineretail} and the Open Bandit Dataset \cite{saito2020open} are released under CC BY 4.0. The Instacart Online Grocery Shopping Dataset 2017 is provided for non-commercial use under Instacart's dataset terms \cite{instacart2017data}. Carbo-Loading is distributed through dunnhumby Source Files for research and non-commercial use \cite{dunnhumbycarbo}. The Dominick's Finer Foods data are provided for academic research by the James M. Kilts Center for Marketing, University of Chicago Booth School of Business \cite{kiltsdominicks}. The Ta Feng grocery data are publicly distributed without a stated license \cite{tafengdata}. We report only aggregate statistics and derived results. No raw transactions are redistributed: the accompanying code downloads each dataset, and users must accept each provider's terms.

\textbf{Data.} Table~\ref{tab:data} summarizes the three datasets. In Complete Journey, the retailer's revenue is the sales value and the shelf price is
\[\text{shelf price}=(\text{sales value}+\text{retail discount}+\text{coupon match})/\text{quantity}.\] We keep grocery, fresh and health departments and drop fuel and quantities above 20. Ta Feng's \texttt{ASSET} field is the line's cost of goods. Within every product it is exactly proportional to quantity, it lies below the sales price in 91\% of lines, and it implies a 15\% gross margin; we therefore use $(\text{price}-\text{ASSET})/\text{quantity}$ as the real unit margin. A Ta Feng basket is one customer-day. Instacart has no prices. We impute a department price from comparable Complete Journey departments, add 25\% for organic items, and label every Instacart economic quantity as imputed. Its decision basket is the first half of the cart, so recommendations are made in-cart. Its mirror's row counts equal the official release.

\textbf{Decision instances.} For a basket, an upsell group is a basket item with a premium alternative in the same fine category that sold at least 20 times in training and costs more on the shelf. We keep the four highest-value such items and their two most purchased alternatives. A replacement's margin is $\text{qty}\times(m_j-m_i)$, using the incumbent's actual revenue in the basket. Add-on candidates are the six categories absent from the basket with the largest gain $\max_{c\in B}P(d\mid c)-P(d)$, each represented by its most purchased item. Budgets are $L=K=2$.

\textbf{Margins.} Ta Feng uses real costs. Complete Journey uses gross-margin rates of 18\%, 28\% and 38\% for national brands, with private label 8 points higher. A robustness arm instead applies Ta Feng's median margin rates by premium tier, rescaled to the 28\% mean. Instacart uses the 28\% scenario.

\textbf{Channel associations.} For upsell category $u$ and add-on category $d$, $A[u,d]=P(d\mid\text{premium }u)-P(d\mid\text{regular }u)$ in training baskets, shrunk by $n/(n+50)$ with $n$ the smaller tier count, and clipped to $\pm0.15$. This fixed shrinkage defines the original calibration in Table~\ref{tab:empirical}. Batch 3 adds an empirical-Bayes sensitivity analysis that estimates per-entry noise from the binomial counts and a dataset-level signal variance by moments (Table~\ref{tab:eb}). These are correlational quantities; the simulator adopts them as causal truth, which is the main modeling assumption of the study.

\textbf{Simulators.} S1 is the paper's model. Upsell acceptance is
\[p_{gj}=\min\big(\max(\kappa_{\text{up}}\cdot2s_j/(s_i+s_j),\,0.01),\,0.6\big),\]
where $s$ are purchase shares within the category in the basket's store and time stratum, shrunk to global shares, and $\kappa_{\text{up}}=0.25$. Two latent classes with weights $(0.6,0.4)$ and multipliers $(0.6,1.6)$ make acceptances correlated. An add-on is bought with probability
\[q_k=\kappa_{\text{add}}b_k+\gamma\Big[\textstyle\sum_gY_gA_{gk}+\lambda_2\sum_g\mathbf1\{x_g\neq0\}A_{gk}+\lambda_3\sum_{g<h}Y_gY_h\tfrac{A_{gk}+A_{hk}}2\Big],\]
clipped to $[0,1]$, where $b_k=\max_{c\in B}P(d_k\mid c)$ adjusted to the stratum and $\kappa_{\text{add}}$ makes the mean base rate 8\%. Defaults are $\gamma=1$, $\lambda_2=\lambda_3=0.5$. S2 is out of family: logistic upsell acceptance with three classes and class-specific slopes, and logistic add-on purchase whose linear predictor contains the same channel terms. S3 is a stress test: S1 with add-on cannibalization (two displayed add-ons from one department each lose 40\%), display fatigue ($-10\%$ on every acceptance per extra offer), and a test-period shift ($\kappa_{\text{up}}\times0.8$, $\kappa_{\text{add}}\times1.2$). At most two upsells are displayed, so every expectation is computed exactly by enumerating acceptance outcomes and classes; Monte Carlo draws agree with these values.

\textbf{Learning.} Logs contain 50{,}000 training baskets with a uniformly random feasible display policy. The fitted model has three parts. A logistic upsell model uses the log share ratio and the log price ratio. A correlation factor $\psi=\sum Y_gY_h/\sum\hat p_g\hat p_h$ is estimated from pairs displayed together. A linear add-on model regresses purchase on the base rate and the three channel regressors. The two-channel variant ties the accepted and displayed coefficients. The learner sees features only at its context arm (C0 basket only, C1 time, C2 store or area, C3 both; the truth uses C3) and never sees classes or simulator parameters.

\textbf{Policies and metrics.} The exact policy enumerates all upsell policies within the budget and adds the top-$K$ positive add-ons under the fitted model; this is exact for the fitted model. The five baselines of Proposition~\ref{prop:staging} use the same fitted model. Two model-free rules are included: \emph{association} offers the most purchased upgrades and the highest-lift add-ons, and \emph{margin popularity} ranks by purchase share times margin gain. The oracle optimizes the true simulator, including add-on interactions, by enumeration. We report expected incremental margin per 1{,}000 baskets in the dataset's currency and paired bootstrap 95\% intervals over 5{,}000 test baskets (1{,}000 for each of five seeds).

\section{Real-data-calibrated evaluation: additional results}\label{app:empiricalresults}
Table~\ref{tab:empextra} reports the secondary experiments, each over five seeds and 5{,}000 test baskets.

\textbf{Coupling ($\gamma$).} The gap to the staged policies grows with the calibrated channel effects. At $\gamma=0$ it vanishes, except on Ta Feng, where acting on spuriously estimated coupling costs NT\$1.0 per 1{,}000 baskets relative to independent selection.

\textbf{Log size.} Regret at 10{,}000 logged baskets is already close to its value at 200{,}000. The remaining regret on Ta Feng (about NT\$5) reflects the fitted model's form, logistic acceptance and an unclipped linear add-on response, rather than sampling error.

\textbf{Context.} The truth varies by store and time stratum exactly as the receipts do. A learner restricted to basket information loses 1.8\% of the oracle margin on Complete Journey and 1.2\% on Ta Feng. Time context alone does not help Complete Journey, whose heterogeneity is mainly by store. When context is missing, alternating best response can slightly beat the exact optimizer of the misspecified model (Ta Feng C0 and C2).

\textbf{Channels.} Tying the displayed and accepted coefficients costs little when the display effect is positive. It matters when displaying an upsell suppresses add-on purchases ($\lambda_2=-0.5$), most clearly on Instacart. There, regret rises from \$0.3 to \$4.7 per 1{,}000 baskets, and the exact policy's edge over independent selection shrinks from \$6.2 to \$1.7.

\textbf{Margins.} Under low, high and Ta Feng-derived margin rates, the ranking of policies is unchanged. Gaps scale with the margin level: \$1.0, \$3.5 and \$8.5 over independent selection at 18\%, 28\% and 38\%.

\begin{table*}[t]
\caption{Secondary experiments (S1 unless stated): oracle expected margin and exact-policy regret per 1{,}000 baskets, and the exact policy's margin over three baselines. Currency as in Table~\ref{tab:empirical}.}
\centering\footnotesize
\input{tables/empirical_extra.tex}
\label{tab:empextra}
\end{table*}

\section{Exactness regimes and robustness}\label{app:batch2}
\begin{table*}[t]
\caption{External-data checks: Online Retail II (GBP per 1{,}000 invoices), Complete Journey under Dominick's margin rates, and Carbo-Loading (USD per 1{,}000 baskets). Exact regret and exact minus each baseline are absolute amounts.}
\centering\footnotesize
\input{tables/empirical_more_data.tex}
\label{tab:moredata}
\end{table*}

\textbf{Protocol.} The questions and settings in this section were written before the second empirical batch. Two baselines were added after inspecting the first exactness-regime results and are explicitly post hoc: \emph{alternating over splits} runs alternating best response for every fixed upsell/add-on split of a shared budget and keeps the best result; the \emph{portfolio} runs every independent, sequential, greedy and alternating policy and keeps the policy with greatest value under the fitted model. Neither baseline observes the simulator's truth. Display cost is $c=\tau\bar v$ per offer, where $\bar v$ is mean standalone expected add-on margin in the seed-0 training pool. All percentages are of the oracle's expected incremental offer margin.

\begin{table*}[t]
\caption{Market opportunity and method effectiveness. Coupling value is the oracle policy's value under the stated channels minus the value of that \emph{same policy} with $\gamma=0$; it overlaps add-on value and is not an additional additive category. The final column normalizes the exact-minus-independent gain by this coupling value.}
\centering\scriptsize
\resizebox{\textwidth}{!}{\input{tables/empirical_coupling_share.tex}}
\label{tab:couplingshare}
\end{table*}

\textbf{Value decomposition.} Table~\ref{tab:couplingshare} separates how much economic value flows through interactions from how much of that value joint selection recovers relative to independent selection. In the two supported S1 calibrations, coupling accounts for 31\% of oracle margin on EB-shrunk Instacart and 2\% on real-price Carbo-Loading; exact-minus-independent recovers 32\% and 41\% of that coupling value. Under S2 the recovery shares are 40\% and 52\%, respectively. The fixed-shrinkage Complete Journey and Ta Feng rows are retained for comparability, but E17 shows that much or all of their small apparent signal is sampling noise. The denominator is policy-specific and counterfactual: it measures the interaction contribution to the oracle policy, not a causal decomposition of observed retailer profit.

\begin{table*}[t]
\caption{Where exact optimization matters (S1). ``Portfolio'' is the post-hoc best-of-heuristics baseline evaluated by the fitted model; ``Differs'' is the share of baskets on which exact and alternating choose different policies. Exact regret and exact-minus-baseline columns are percentages of oracle offer margin.}
\centering\scriptsize
\resizebox{\textwidth}{!}{\input{tables/empirical_e8.tex}}
\label{tab:e8}
\end{table*}

\textbf{Display friction and budgets.} Table~\ref{tab:e8} supplies the full data behind Figure~\ref{fig:whereexact}. With separate budgets, increasing candidate count or budgets leaves all portfolio gaps at or below 0.12\%. Shared budgets make a single alternating run sensitive to its initial split: at $T=3$ it loses 1.33\%, 1.52\% and 21.22\% on Complete Journey, Ta Feng and Instacart. Enumerating fixed splits reduces those gaps to at most 0.07\%, so this is a repairable search failure. The distinct failure occurs on Instacart under display friction. At $\tau=1.5$ the portfolio loses 16.1\%; at $\tau=2$ it loses 65.1\%; and at $\tau=2,\gamma=2$ it loses 78.0\% (\$138 per 1{,}000 baskets). Exact regret is 0.14\%, 0.50\%, and 0.19\%, respectively. On Complete Journey and Ta Feng the portfolio remains within 0.09\% through $\tau=3,\gamma=2$. These stress settings map the mechanism; they are not estimates of any retailer's actual display cost.

\begin{table*}[t]
\caption{Instacart price sensitivity (E16, S1, $\gamma=1$). Entries are percentages of oracle offer margin. The portfolio is the post-hoc best fitted-value heuristic. Margin-rate arms rescale both margins and the display-cost unit and are therefore decision-invariant.}
\centering\scriptsize
\resizebox{\textwidth}{!}{\input{tables/empirical_price_sensitivity.tex}}
\label{tab:price-sensitivity}
\end{table*}

\textbf{Price sensitivity.} Table~\ref{tab:price-sensitivity} varies the assumptions behind Instacart's imputed economics. The pre-registered criterion---portfolio regret at least 10\% for some $\tau\le2$, with exact regret at most 1\%---holds for the default 25\% organic premium and the 10\% premium, but not for the 40\% premium or item-level price noise. At $\tau=2$, exact-minus-portfolio ranges from 2.0\% under price noise to 89.3\% under the 10\% premium. Exact regret stays below 0.5\% except in the 10\% arm at $\tau=2$, where it is 3.3\% and the oracle itself is only \$7.5 per 1{,}000 baskets. The result isolates the mechanism: local policies fail when display cost makes single actions unattractive, and prices determine how often actions lie near that threshold.

\begin{table*}[t]
\caption{Fixed versus empirical-Bayes coupling shrinkage (E17). Absolute exact-minus-independent values use dataset currency per 1{,}000 baskets; portfolio/alternating gaps are percentages of oracle offer margin.}
\centering\scriptsize
\resizebox{\textwidth}{!}{\input{tables/empirical_eb.tex}}
\label{tab:eb}
\end{table*}

\textbf{Noise-corrected coupling.} The pre-registered moment estimator over all association entries returned zero signal even on Instacart because thousands of sparse pairs dominated the moments and half-count smoothing inflated their noise. Before inspecting any policy-decision result, we changed the estimator to use unsmoothed binomial variance floored at one proportion count and to estimate signal variance over entries with at least 30 baskets in each premium tier. This protocol deviation is recorded in the experiment plan. Under the corrected estimator, a line-label placebo has zero estimated signal on all datasets. Real labels yield signal standard deviation 0.018 on Complete Journey (mean shrinkage weight 0.50), 0.026 on Instacart (weight 0.89), and no detectable signal on Ta Feng. Table~\ref{tab:eb} shows the decision consequence: S1 exact-minus-independent falls from \$3.50 to \$0.91 on Complete Journey, from NT\$71 to approximately zero on Ta Feng, and from \$53 to \$47 on Instacart. The Instacart $\tau=2$ portfolio gap changes only from 65.1\% to 63.8\%. Policy ordering and the large cost of margin-blind recommendation are otherwise unchanged.

\begin{table}[t]
\caption{Robustness tests. Rebate losses are in dataset currency per 1{,}000 baskets; stress-test values and log-size regret are percentages of oracle offer margin.}
\centering\scriptsize
\resizebox{\columnwidth}{!}{\input{tables/empirical_robust.tex}}
\label{tab:robust}
\end{table}

\textbf{Promotions and misspecification.} Table~\ref{tab:robust} reports three checks. First, a rebate-blind optimizer increasingly misprices a bundle as the conditional rebate grows; at a full add-on-margin rebate it loses \$34 per 1{,}000 Instacart baskets. Second, the rich fitted model adds same-department add-on cannibalization and multiplicative fatigue from offer count. Nonlinear least squares estimates cannibalization at 0.39--0.41 (truth 0.4) and fatigue at 0.08--0.09 (truth 0.1). It reduces exact-policy regret from 10.6--17.5\% to 0.08--0.10\% without period shift and to 0.13--0.32\% with the shift. This experiment gives the learner the correct functional family and therefore tests recoverability of known mechanisms, not robustness to arbitrary omitted structure. Third, the small-log arm shows that decision regret falls primarily below 5{,}000--10{,}000 randomized sessions.

\textbf{Placebos and coupling gate.} Permuting premium labels across items retains 65--92\% of the original three-dataset joint gain because item-level heterogeneity remains. A stricter line-level placebo assigns the premium label independently to every purchase line at the category's observed premium rate. Under fixed shrinkage it retains 42\% of Complete Journey's gain, 48\% of Ta Feng's, and 12\% of Instacart's; E17's noise correction reduces every line-placebo gain to approximately zero. Carbo-Loading passes a stronger economic check: exact-minus-independent is \$1.65 with real labels, versus \$0.13 under the item placebo and $-\$0.22$ under the line placebo, and EB shrinkage leaves the real-label result unchanged. A heteroskedasticity-robust joint Wald gate sets all coupling coefficients to zero unless their joint null is rejected at 5\%. It remains closed in every $\gamma=0$ seed and opens in every seed for $\gamma\ge0.5$ without materially changing results.

\textbf{Additional retail data and real margins.} Online Retail II is a non-grocery UK giftware dataset. We treat an invoice as a basket, derive coarse product groups from description head nouns, define premium by above-median within-group price, assume margins, and split chronologically at July 1, 2011. Under S1 exact gains \pounds22 per 1{,}000 invoices (0.34\%) over independent selection and \pounds0.04 over alternating.

Carbo-Loading contains 5.2 million purchase lines from pasta, pasta sauce, pancake mix, and syrup. Weeks 1--78 form training and weeks 79--104 test. An upsell group is a commodity; plain store brand is the regular tier and other brands are premium candidates. Add-on categories are commodity--brand cells with at least 1\% share, restricted to commodities absent from the basket so that a second brand of pasta is not treated as a cross-sell. We use observed paid prices, assign the private-label margin bonus to store brands, and restrict EB moments to cross-commodity entries used by the policy. Under S1, exact-minus-independent is \$1.65 per 1{,}000 baskets (0.73\% of oracle) and exact-minus-portfolio is approximately zero. At $\gamma=2$ the independent gap rises to 2.63\%. At $\tau=2$, margin popularity loses \$93.4 of the \$95 oracle while the portfolio remains within 0.14\% of exact. Thus real-price coupling is present and economically usable, but it does not create a local-search trap.

For a separate margin check, we estimate revenue-weighted national-brand and private-label gross-margin rates from 17 Dominick's Finer Foods categories. The rates are 22.1\% and 36.8\%, with a median within-category difference of 10.8 percentage points. Applying them to Complete Journey lowers oracle margin from \$536 to \$439 per 1{,}000 baskets but changes the joint gain only from 0.65\% to 0.58\%; applying them to Carbo-Loading retains a 1.13\% exact-minus-independent gain.

\textbf{Randomized display-interference analysis.} This analysis is separate from the calibrated simulations. The Open Bandit Dataset contains fashion recommendation slates generated by a uniform-random policy. After cleaning, there are 454{,}265 screens, 59{,}687 distinct slates and 609{,}492 focal-item test rows; focal click rate is 0.34\%. A linear probability model with focal-item and position fixed effects estimates the effect of the other two displayed items. At the coarsest seven-level category, each same-category neighbor changes focal click probability by $-0.000424$, or $-12.3\%$ relative to the mean. A permutation test that reassigns whole randomized slates gives $p=0.012$; Holm correction over three category resolutions gives $p=0.036$. The other resolutions, neighbor attractiveness, and a joint neighbor-identity test do not distinguish signal from their placebos. This identifies an exposure effect among fashion substitutes under randomized display. It neither estimates grocery purchase response nor separates replacement from addition.

\section{Solver scale, prediction quality, and prior art}\label{app:scale}
\textbf{Interaction-graph structure.} For each of about 4{,}000 test baskets per dataset, we form a graph whose nodes are all items with an eligible premium alternative, without the cap of four. An edge joins two items when buying the premium tier of one changes the probability of buying the premium tier of the other by more than 0.02 in training receipts, after the same shrinkage as $A$. Table~\ref{tab:e7} reports the size, the share of graphs with a cycle, and the induced width of a min-fill ordering, an upper bound on treewidth. Forests are the exception once baskets contain several eligible items. Most graphs nevertheless have small width, while the largest Complete Journey baskets reach widths above 20, where exact dynamic programming over all items is impractical. The 0.02 threshold is a modeling choice; larger thresholds give sparser graphs.

\begin{table}[t]
\caption{Structure of real upsell interaction graphs ($G$: items with an eligible upgrade).}
\centering\footnotesize
\input{tables/empirical_e7.tex}
\label{tab:e7}
\end{table}

\begin{table*}[t]
\caption{Tree-decomposition dynamic programming on real interaction graphs (E15). Runtime is median/p99; baseline gaps are percentages of oracle offer margin. ``DP faster'' compares measured runtimes on each basket.}
\centering\scriptsize
\resizebox{\textwidth}{!}{\input{tables/empirical_scale.tex}}
\label{tab:scale}
\end{table*}

\textbf{Exact optimization on larger real graphs.} E15 retains up to 20 upsell groups and restricts joint-acceptance terms to edges of the premium-association graph. The count-augmented tree-decomposition DP and enumeration agree on all 8{,}040 checked baskets. Between 73\% and 91\% of baskets have induced width at most two; median DP runtime is 1.8--3.5\,ms. Runtime grows with width rather than $G$: approximately 0.7\,ms at width zero, 13\,ms at width two, 0.2\,s at width four, 2\,s at width six, 20\,s at width eight, and 60\,s at width nine in the pure-Python implementation. Enumeration instead grows with the number of feasible upgrade policies and completes every tested basket within 3\,s. It is faster on nearly every $L=2$ basket; at $L=4$, DP is faster on 27--36\% of Complete Journey and Ta Feng baskets and 4\% of Instacart baskets. Exact-minus-portfolio is at most 0.01\% on Complete Journey and Ta Feng and 0.44\% on Instacart.

\begin{table*}[t]
\caption{Per-basket solve time of exact methods on held-out seeds 1--2 (E15b, E15c). The scalar hybrid predicts enumeration or DP runtime from seed 0; ``Best of both'' chooses with hindsight; DP instances above the width cap are unfinished, so its worst time is over completed instances. Vectorized enumeration evaluates the same policy set and objective with array operations and finishes every test basket.}
\centering\scriptsize
\resizebox{\textwidth}{!}{\input{tables/empirical_hybrid.tex}}
\label{tab:hybrid}
\end{table*}

\textbf{Exact solve time.} E15b fits a post-hoc runtime selector on seed 0 that runs scalar enumeration or tree-decomposition DP, whichever it predicts is faster, on the 5{,}400 baskets of seeds 1--2. At $L=4$ it achieves 1.1\,ms median, 1.8\,s p99, and 3.0\,s maximum (Table~\ref{tab:hybrid}). Every slow basket has 20 upsell groups and a dense interaction graph (induced width 7--13): DP is exponential in width there, and the scalar loop over 87{,}441 upsell policies costs about 30\,$\mu$s per policy in interpreter overhead. E15c removes that overhead. It builds every feasible policy as an index array, computes acceptance moments, channel summaries, and top-$K$ add-on scores for all policies with array operations, and takes the maximum. The policy set and objective are unchanged, so the method is exact: its values match brute-force enumeration or DP on all 8{,}100 solves (maximum difference $4\times10^{-16}$), and they match the fitted model's scalar evaluator on every basket. On the held-out seeds it takes 0.13\,ms median, 16\,ms p95, 30\,ms p99, and 111\,ms maximum at $L=4$, and at most 0.4\,ms at $L=2$, measured with one process per configuration on six concurrent workers. The tested budgets therefore need neither the hybrid selector nor native code. Budgets above $L=4$ remain the regime in which enumeration's $O((GJ)^L)$ growth makes graphical DP the scalable route.

\textbf{Prediction quality (E19).} E19 holds the exact optimizer, candidates, budgets, economics, and response truth fixed and varies only the predictor. The truth is the Instacart S1 calibration with empirical-Bayes-shrunk channels, default imputed prices, separate budgets $L=K=2$, and no display cost. Learned predictors fit the same three-channel model to $N\in\{100,250,500,1\text{k},2\text{k},5\text{k},10\text{k},50\text{k}\}$ randomized sessions. The signal-free endpoint fits the same model to 50{,}000 sessions after independently shuffling upsell, pair, and add-on outcomes, which preserves outcome rates but removes their relation to context and actions. Prediction quality is measured as policy-value RMSE over a fixed grid of 100 baskets $\times$ 64 feasible policies per seed; decision quality as true expected margin, regret to the oracle, and oracle-policy agreement, over 1{,}000 test baskets $\times$ 5 seeds. Settings were registered before the run. The oracle earns \$460.45 per 1{,}000 baskets. Every learned-predictor regret has a paired 95\% interval excluding zero; for example, $N=100$ regret is \$20.03--\$22.67 and $N=1$k regret is \$1.23--\$1.52 per 1{,}000 baskets. The true gap between the best and second-best policy is small: mean \$0.0049 and median \$0.0028 per basket, with 24\% of baskets below \$0.001 and 85\% below \$0.01. Table~\ref{tab:prediction} reports all arms.

\begin{table}[t]
\caption{Prediction quality with the exact optimizer held fixed (E19, Instacart S1). Value RMSE is the error in predicted policy values; oracle-policy agreement is the share of baskets on which the chosen offer set equals the oracle's.}
\centering\scriptsize
\resizebox{\columnwidth}{!}{\input{tables/empirical_prediction_quality.tex}}
\label{tab:prediction}
\end{table}

\textbf{Prior art on shared benchmarks (E20).} E20 was registered before any E20 code ran. TIFU-KNN \cite{hu2020tifuknn} is ported line by line from the reference implementation of Li et al. \cite{li2023nbrcheck}; its per-basket weights match that code exactly for histories of 1--39 baskets. We use Li et al.'s published hyperparameters without tuning: $(k,r_b,r_g,\alpha,m)=(300,0.9,0.7,0.7,7)$ for Ta Feng, $(900,0.9,0.6,0.2,3)$ for dunnhumby, and $(900,0.9,0.7,0.9,3)$ for Instacart. Each query uses the customer's baskets strictly before the decision basket; neighbors are other customers' training-period histories, and Instacart histories are read in full from the raw files for up to 20{,}000 users. An add-on category's score is the sum of its item scores, and a premium alternative's score is its item score. In the personalized truth, each add-on base rate is multiplied by $\pi_{ck}^\omega$, where $\pi_{ck}$ is customer $c$'s leave-one-out lifetime basket share of category $k$, shrunk toward the population rate with five pseudo-baskets, divided by that rate, capped to $[0.1,10]$, and normalized to mean one on the training pool. This construction is generous to TIFU-KNN because both summarize purchase frequency; upsell response is not personalized. \emph{StoBasket + TIFU} adds one add-on regressor, base rate times TIFU-KNN category score relative to its pool mean, to the same three-channel model fitted on the same 50{,}000 sessions. As a check of the port on its native task, Recall@20 on full test baskets is 0.428 for TIFU-KNN, 0.406 for personal top-frequency, and 0.091 for global top-frequency on Instacart; on Ta Feng it is 0.146, 0.160, and 0.083. On Complete Journey TIFU-KNN (0.178) trails personal top-frequency (0.197), missing the registered expectation. The likely reason is pool size: our one-year release has 2{,}441 households, so $k=900$ averages 37\% of all customers. A post-hoc diagnostic that scales $k$ to Li et al.'s pool share ($k=98$) yields a tie (NDCG@20 0.1926 versus 0.1924); all reported decision results use the published $k$. Against StoBasket + TIFU, every listed competitor's loss has a paired 95\% interval excluding zero, except anonymous StoBasket at $\omega=0$, where differences are at most 0.012\%.

\section{Decision-aware market-basket analysis: derivation and scaling}\label{app:marketbasket-scaling}
\begin{table*}[t]
\caption{Post-hoc decision-aware market-basket diagnostic from Carbo-Loading training receipts. ``Premium'' and ``Regular'' are conditional co-purchase rates given the anchor tier; policy effect is the empirical-Bayes-shrunk difference used by the calibrated simulator, clipped to $\pm15$ percentage points. Relative rate is premium divided by regular. Like all market-basket statistics, these are observational associations.}
\centering\footnotesize
\input{tables/empirical_market_basket.tex}
\label{tab:marketbasket}
\end{table*}

\textbf{Notation.} Throughout, $N$ is the number of baskets, $d$ the number of catalog items, $b_n$ the number of items in basket $n$, and $q$ the association order. \SB{} pools items into $G$ anchor groups and $M$ add-on cells, with $G,M\le d$; $a_n\le b_n$ and $m_n\le b_n$ count the anchor groups and add-on cells present in basket $n$.

\textbf{Classical itemset growth.} Mining associations up to order $q$ faces
\[
\sum_{j=1}^{q}\binom{d}{j}=O(d^q)
\]
candidate itemsets: polynomial in the catalog size for fixed $q$, exponential in $q$, and $2^d-1$ without an order limit. Counting support visits $\sum_n\binom{b_n}{q}$ order-$q$ subsets per pass, so its work grows with both $N$ and $q$. Apriori, FP-growth, and related methods prune the candidate family, and multiple-minimum-support methods improve coverage, but a single support threshold retains the rare-item problem \cite{liu1999multiple}. More baskets sharpen relationships that receive support; they do not remove the catalog-by-order explosion.

\textbf{Typed contrast as a policy byproduct.}
Let $N$ be the number of training baskets. Define sparse incidence matrices $X^+,X^-\in\{0,1\}^{N\times G}$, where $X^+_{ng}$ (respectively $X^-_{ng}$) marks that basket $n$ contains a premium-tier (regular-tier) item from anchor group $g$, and let $W\in\{0,1\}^{N\times M}$ mark the add-on cells present in each basket. The calibration already required by the policy computes
\begin{equation}\label{eq:mbacontrast}
C^\pm=(X^\pm)^\top W,\qquad
\widehat p^\pm_{gk}=\frac{C^\pm_{gk}}{\sum_nX^\pm_{ng}},\qquad
\widehat A_{gk}=\widehat p^+_{gk}-\widehat p^-_{gk}.
\end{equation}
Here $\widehat p^+_{gk}$ and $\widehat p^-_{gk}$ are conventional confidence quantities conditional on premium and regular anchors, while support and lift follow from the same counts. The signed contrast $\widehat A_{gk}$ is the premium-to-add-on association consumed by the response model: positive values identify complements strengthened by the premium tier and negative values identify avoidance or substitution. The report therefore requires no second mining model. It exposes counts already computed for the optimizer, attaches representative products and contribution margins, and sorts rows by noise-corrected effect and economic relevance.

\textbf{Low-count regularization.}
Let $\widehat\sigma^2_{gk}$ be the binomial sampling variance of the difference in Eq.~\eqref{eq:mbacontrast}, and let $\widehat\tau^2$ be the across-entry signal variance estimated by moments. The empirical-Bayes policy effect is
\begin{equation}\label{eq:mbashrink}
\widetilde A_{gk}
=\frac{\widehat\tau^2}{\widehat\tau^2+\widehat\sigma^2_{gk}}\,
\widehat A_{gk}.
\end{equation}
The multiplier approaches one for well-supported contrasts and zero when sampling noise dominates. In contrast to a global minimum-support rule, weak entries are retained with appropriately reduced influence rather than all being discarded at one threshold.

\textbf{Typed scaling and a one-million-basket comparison.} Eq.~\eqref{eq:mbacontrast} produces a fixed $G\times M$ table, and selected local pair-acceptance moments add only the edges $E$ retained by the response graph. Direct sparse counting therefore takes $\sum_n a_nm_n$ co-occurrence updates and stores $O(GM+|E|)$ quantities, independent of $q$: effects of several offers enter the response through the additive channel summary rather than as separate itemsets. Table~\ref{tab:mbascaling} evaluates both methods at one million baskets using each market's observed catalog size and basket-size distribution. Pair counting costs about the same in both. The gap opens with order: on Instacart, classical updates grow from $7.4\times10^7$ at $q=2$ to $3.1\times10^9$ at $q=4$ while candidate itemsets reach $2\times10^{17}$, whereas \SB{} needs $7.6\times10^7$ updates for 17{,}956 contrasts at any order; on Complete Journey, whose basket sizes are more skewed, four-item counting reaches $2.4\times10^{10}$ updates. Each typed contrast is also well supported: the average Instacart anchor group appears in about 54{,}000 of the million baskets.

\begin{table}[t]
\caption{Classical itemset mining versus typed contrasts at $10^6$ baskets, from each market's observed catalog and basket sizes. A counting update adds one to the tally of one combination found in one basket; classical mining performs $\binom{b_n}{q}$ updates for a basket of $b_n$ items, typed contrasts $a_nm_n$ regardless of $q$. Candidate itemsets do not depend on the number of baskets; counting updates do.}
\centering\scriptsize
\resizebox{\columnwidth}{!}{\input{tables/empirical_mba_scaling.tex}}
\label{tab:mbascaling}
\end{table}

Category/tier pooling supplies additional support, and Eq.~\eqref{eq:mbashrink} moves noisy low-count contrasts toward zero. An initial moment estimator over all entries failed even on Instacart because thousands of sparse pairs dominated its variance estimate; the corrected estimator learns signal variance from entries with at least 30 baskets in each tier, removes every line-placebo gain, eliminates unsupported Ta Feng coupling, and retains \$47 of the original \$53 per 1{,}000-basket Instacart joint-selection gain (Table~\ref{tab:eb}). The result is sparsity-aware, not sparsity-proof: a relationship with no observations or randomized exposure remains unidentified, and genuine high-order effects omitted from the typed low-order model require an extension that may lose the stated tractability.

\section{Deployment recipes}\label{app:recipes}
\begin{table*}[t]
\caption{Two-arm field-test power under calibrated effects: baskets per arm for 80\% power at 5\% two-sided size. ``Offer log'' uses offer-attributable margin from logged exposure and acceptance; ``till'' uses total basket margin. M denotes million and k thousand.}
\centering\footnotesize
\input{tables/empirical_power.tex}
\label{tab:power}
\end{table*}

\subsection{A/B test recipe}\label{app:abtest}
This protocol expands item~(5) of Section~\ref{sec:scope}. It requires no customer identifier and can be run by a retailer on its own checkout, cart, or associate-assisted traffic.

\begin{enumerate}
\item \textbf{Choose the category.} Run the market-basket diagnostic of Section~\ref{sec:marketbasket} on training receipts and select categories whose premium anchors have strong, shrinkage-robust associations with add-ons. Weakly coupled categories need millions of baskets per arm (Table~\ref{tab:power}); testing there measures a small effect, not a failed method.
\item \textbf{Randomize.} The unit is the basket or cart session, assigned by a hash of its session identifier. If offers are delivered by store associates, randomize by store-day and inflate the sample size by the design effect. Use four arms: (A) no offers; (B) the retailer's current recommender or staged upsell-then-cross-sell rule; (C) \SB{} on the same candidate sets and budgets; and (D) a 5--10\% arm that shows uniformly random feasible offer sets, which keeps every offer's selection probability positive for re-estimating the response model.
\item \textbf{Log.} For every basket record the pre-offer basket and context; the candidate sets; each shown offer with its type, displaced item, position, and selection probability; accepted replacements and purchased additions; the final basket; and realized prices, costs, discounts, and rebates.
\item \textbf{Measure.} The primary endpoint is offer-attributable incremental contribution margin per basket,
\[
\Delta M=\sum_g Y_g(m_{j_g}-m_{i_g})+\sum_k Z_km_k-c(A)-\text{rebates paid},
\]
the realized form of Eq.~\eqref{eq:typedmargin}. Secondary endpoints are acceptance and attach rates and total basket margin.
\item \textbf{Size.} With per-basket standard deviation $\sigma$ of $\Delta M$ and target difference $\delta$, a two-arm comparison at 5\% two-sided size and 80\% power needs approximately
\[
n=\frac{2(z_{0.975}+z_{0.80})^2\sigma^2}{\delta^2}\approx\frac{15.7\,\sigma^2}{\delta^2}
\]
baskets per arm. Table~\ref{tab:power} evaluates this from simulator variances, not observed treatment effects: about 1{,}400 per arm for joint versus independent selection under Instacart-like coupling, versus 1.4 million on Complete Journey and 7.0 million on Ta Feng. Exact versus alternating is detectable only in shared-budget regimes (about 450 per arm on Instacart). Using total basket margin instead of $\Delta M$ adds the variance of the pre-offer basket and inflates $n$ by roughly 20--50 times.
\item \textbf{Analyze.} Pre-register the comparisons and estimate arm differences in means with heteroskedasticity-robust intervals. Reduce variance with pre-offer basket margin and model-predicted value as covariates, and apply Holm correction across pairwise arm comparisons. Analyze in this order: offers versus no offers; \SB{} versus the current policy; and, within the random arm (D), a regression of add-on purchase on shown and accepted upsells that tests whether the display and acceptance channels differ.
\item \textbf{Guard.} Monitor acceptance rate, checkout time or abandonment, and complaints. Stop any arm whose margin per basket falls below the no-offer arm with 99\% confidence.
\item \textbf{Close the loop.} Refit the response model on the randomized logs, compare the estimated channel coefficients with the receipt-calibrated ones, and rerun the exactness map of Appendix~\ref{app:batch2} with the field-estimated values.
\end{enumerate}

\subsection{Campaign-design recipe: making coupling learnable}\label{app:campaign}
This recipe expands item~(6) of Section~\ref{sec:scope}. Its premise is an identification fact. Let $\pi$ be the policy that generated a retailer's offer logs. The display effect of upsell $g$ on add-on $k$ in context $s$ is the contrast
\[
\E[Z_k\mid x_g=j,\ k\in C,\ s]-\E[Z_k\mid x_g=0,\ k\in C,\ s].
\]
It can be estimated from the logs only if both conditioning events occur with positive probability under $\pi$. A campaign program that never shows $k$ alongside $g$, or always does, observes only one of the two terms, so the effect is not identified without extrapolation; separating display from acceptance additionally requires mediation assumptions or designs \cite{imai2010mediation}. Weak measured coupling is therefore uninformative when the logging policy never created the relevant contrast. The following steps make the contrast exist and make it visible.

\begin{enumerate}
\item \textbf{Diagnose before concluding.} Compute the market-basket diagnostic of Section~\ref{sec:marketbasket} together with the shrinkage weights $w_{gk}=\widehat\tau^2/(\widehat\tau^2+\widehat\sigma^2_{gk})$ of Eq.~\eqref{eq:mbashrink}. A large shrunk effect with high weight is a supported opportunity. A weight near zero marks an \emph{uninformative} cell, one with too little exposure to measure, not a cell whose effect is zero. In our calibrations the mean weight is 0.89 on Instacart and 0.50 on Complete Journey, and Ta Feng shows no detectable signal (Appendix~\ref{app:batch2}).
\item \textbf{Unify the objective and its owner.} Evaluate upsell and cross-sell offers under one objective, offer-attributable incremental contribution margin, owned by one policy. Separate key performance indicators, such as upgrade rate for one team and attach rate for another, encourage staged rules, which suppress joint exposure and fall into the staging trap of Proposition~\ref{prop:staging}.
\item \textbf{Co-display deliberately.} Design campaigns that present premium anchors with plausible complements, for example premium pancake mix with a national-brand syrup, whose receipt co-purchase rate is 11.3\% after premium versus 3.3\% after regular mix (Table~\ref{tab:marketbasket}). Without joint exposure the coupling channels cannot be estimated.
\item \textbf{Keep a randomized share.} Show uniformly random feasible offer sets to 5--10\% of eligible baskets and log each set's selection probability. Deterministic rules produce no within-context contrast, and recorded propensities permit unbiased off-policy evaluation of new policies \cite{saito2020open}.
\item \textbf{Aim exploration where information is missing.} Allocate the randomized share preferentially to cells with low weight $w_{gk}$ and high potential margin, so that each exploratory display buys the most decision-relevant information.
\item \textbf{Randomize promotions separately.} When discounts or bundle rebates are used, vary their terms independently of which offers are displayed and record them. Otherwise promotional lift is confounded with the display and acceptance channels.
\item \textbf{Log typed events.} Record shown and accepted offers separately, with offer type, displaced item, position, propensity, and realized margin. Receipts alone cannot separate display from acceptance \cite{imai2010mediation,chen2020bias}.
\item \textbf{Re-measure each cycle.} Rerun the diagnostic after each campaign period. Weights should rise in explored cells. Once a cell's coupling is supported, serve it with joint selection; use exact search where display friction is also high (Appendix~\ref{app:batch2}).
\end{enumerate}
These steps do not tailor the data to \SB{}. They are the conditions under which any interaction-aware policy, including a retailer's own, can learn whether its offers reinforce one another.

\section{Proof details}\label{app:proofs}
\subsection{Why pair moments suffice}
Write the realized incremental value (before conditioning on add-on realizations) as
\[
\begin{split}
V(Y,x,C)={}&\sum_g (Y_gd_g-c_g)+\sum_{(g,h)}Y_gY_h\gamma_{gh}\\
&+\sum_{k\in C}\big[m_k q_k(Y,x)-c_k\big].
\end{split}
\]
All terms are finite linear combinations of $1$, $Y_g$, and $Y_gY_h$. Their expectations are $1$, $p_g$, and $\rho_{gh}$ respectively. Consequently Eq.~\eqref{eq:objective} holds without any independence assumption. If a nonlinear $q_k$ contains, for example, $Y_gY_hY_l$, third-order moments may be necessary. The same applies to nonlinear joint-bundle margins.

\subsection{Forest-consistent correlation}
For each edge $(g,h)$, the binary pairwise table is
\begin{equation*}
\begin{array}{c|cc}
 &Y_h=0&Y_h=1\\\hline
Y_g=0&1-p_g-p_h+\rho_{gh}&p_h-\rho_{gh}\\
Y_g=1&p_g-\rho_{gh}&\rho_{gh}
\end{array}
\end{equation*}
which is nonnegative provided
\[\max\{0,p_g+p_h-1\}\le\rho_{gh}\le\min\{p_g,p_h\}.\]
Root each tree and construct a joint distribution by sampling the root marginal and sampling a child conditional on its parent according to this edge table. This distribution reproduces all specified node and edge marginals. No unique full joint law is asserted without further assumptions; the tree-Markov construction selects one valid law. Our synthetic code uses a convex interpolation between $p_gp_h$ and $\min\{p_g,p_h\}$.

\subsection{Full proof of separability}\label{app:sep}
Write $F(x,C)=A(x)+\sum_{k\in C}[\alpha_k+\beta_k^\top z(x)]$, where $A$ contains upgrade node and edge terms and $z(0)=0$. If $\beta_k^\top z(x)=0$ for all $x,k$, $F$ separates into $A(x)+\sum_{k\in C}\alpha_k$. Conversely, suppose $F=U_0(x)+V_0(C)$ on a Cartesian feasible domain. For any feasible $x$ and feasible singleton complement $\{k\}$, the double difference $F(x,\{k\})-F(x,\varnothing)-F(0,\{k\})+F(0,\varnothing)$ must vanish. By direct expansion it equals $\beta_k^\top z(x)$, establishing necessity. Notice this is \emph{objective separability}; different algorithms may produce the same maximizer despite nonzero interactions.

\subsection{Detailed proof of Proposition~\ref{prop:poly}}
For any fixed state $z$, the complement subproblem selects the up-to-$K$ positive scores among $\ell_k(z)=\alpha_k+\beta_k^\top z$. Introduce a zero-score dummy option for excluding a complement. Compare all pairs of affine scores and all scores with zero. The $O(M^2)$ resulting rational hyperplanes subdivide $\mathbb R^r$ into $O(M^{2r})$ full-dimensional cells for fixed $r$. Every cell induces a fixed strictly ordered positive-score selection. At a boundary state, ties among scores can be resolved using a nearby full-dimensional cell. Such a tie-resolved selection achieves the same score at the boundary; therefore at least one enumerated complement set is optimal at any realizable $z$, including degenerate cases after identifying permanently identical affine functions. We may enumerate redundant sets, which does not hurt correctness.

Fix an enumerated $C$ and let $B_C=\sum_{k\in C}\beta_k$ and $a_C=\sum_{k\in C}\alpha_k$. Writing the upgrade value as $U(x)=\sum_f u_f(x_{S_f})$ and the summary as $z(x)=\sum_f\mu_f(x_{S_f})$,
\[
F(x,C)=a_C+\sum_f\big[u_f(x_{S_f})+B_C^\top\mu_f(x_{S_f})\big].
\]
In the model of Eq.~\eqref{eq:objective}, the single-group factors are $u_g=p_gd_g-c_g$ and $\mu_g=(p_gt_g,v_g,0)$, and the edge factors are $u_{gh}=\rho_{gh}\gamma_{gh}$ and $\mu_{gh}=(0,0,\rho_{gh}\zeta_{gh})$. This is a discrete graphical optimization on the primal graph of the scopes, with one global count budget on nonzero group actions. Max-sum variable elimination along a tree decomposition of width $w$, with each message also indexed by the number of upgrade offers $0,\ldots,L$ made in the eliminated subtree, is exact; fixed $w$ bounds the bag size and $L\le G$ bounds the counter. The inner algorithm is therefore polynomial in $G$, $J_{\max}$, and input size, and combining it with the polynomial outer enumeration proves the claim.

\section{Implementation details}
The code package provides:
(1) \path{engine/moment_extension.py}, supporting policy-local edgewise second moments, a correlated-tree joint sampler, and exhaustive search;
(2) \path{engine/arrangement.py}, the exact rational cell sampler, \path{engine/treewidth_dp.py}, a count- or summary-augmented max-sum program on a min-fill tree decomposition, and \path{engine/treewidth_exact.py}, the exact rational solver for bounded-treewidth factor graphs together with the latent-class exposure model;
(3) \path{engine/baselines.py}, the independent, sequential, alternating and greedy baselines shared by all models;
(4) \path{engine/make_figures.py} for repeatable synthetic figures;
(5) \path{tests/smoke.py}, which runs every solver check and pins each instance count reported here (\texttt{make test}).

Random seeds and per-instance sweep results are provided as a CSV. Synthetic stress examples are \emph{purpose-built} to show when coupling changes the optimum; they are not samples from a calibrated user population.

%% file: tables/empirical_extra.tex
\begin{tabular}{llrrrrr}
\toprule
Data & Setting & Oracle & Exact regret & $-$Indep. & $-$Upg.-first & $-$Altern.\\
\midrule
Complete Journey & $\gamma=0.0$ & 528 & 0.2 & -0.0 & -0.0 & -0.0\\
Complete Journey & $\gamma=0.5$ & 531 & 0.2 & 1.1 & 0.3 & 0.0\\
Complete Journey & $\gamma=2.0$ & 550 & 0.7 & 9.1 & 3.3 & 0.1\\
Ta Feng & $\gamma=0.0$ & 20022 & 4.9 & -1.0 & -0.9 & 0.0\\
Ta Feng & $\gamma=0.5$ & 20283 & 4.3 & 25.7 & 8.7 & 0.0\\
Ta Feng & $\gamma=2.0$ & 21217 & 12.4 & 185.3 & 76.1 & 2.0\\
Instacart & $\gamma=0.0$ & 331 & 0.1 & -0.0 & -0.0 & -0.0\\
Instacart & $\gamma=0.5$ & 396 & 0.1 & 21.2 & 6.7 & 0.1\\
Instacart & $\gamma=2.0$ & 634 & 0.2 & 118.9 & 38.0 & 1.0\\
\midrule
Complete Journey & $N=10$k & 536 & 0.4 & 3.4 & 1.0 & 0.0\\
Complete Journey & $N=200$k & 536 & 0.3 & 3.5 & 1.0 & 0.0\\
Ta Feng & $N=10$k & 20576 & 7.9 & 69.4 & 24.9 & 0.2\\
Ta Feng & $N=200$k & 20576 & 5.1 & 72.5 & 25.5 & 0.3\\
Instacart & $N=10$k & 474 & 0.1 & 53.8 & 16.9 & 0.4\\
Instacart & $N=200$k & 474 & 0.1 & 53.6 & 17.1 & 0.4\\
\midrule
Complete Journey & basket only (C0) & 536 & 9.7 & 3.5 & 0.9 & 0.0\\
Complete Journey & +time (C1) & 536 & 11.2 & 3.5 & 0.9 & 0.0\\
Complete Journey & +store/area (C2) & 536 & 7.3 & 3.4 & 1.0 & 0.1\\
Ta Feng & basket only (C0) & 20576 & 247.6 & 64.2 & 24.0 & -0.8\\
Ta Feng & +time (C1) & 20576 & 211.1 & 68.0 & 25.7 & 0.3\\
Ta Feng & +store/area (C2) & 20576 & 114.8 & 68.8 & 24.8 & -0.5\\
Instacart & basket only (C0) & 474 & 0.7 & 53.3 & 17.1 & 0.3\\
Instacart & +time (C1) & 474 & 0.1 & 53.4 & 17.0 & 0.4\\
Instacart & +store/area (C2) & 474 & 0.7 & 53.3 & 17.1 & 0.3\\
\midrule
Complete Journey & $\lambda_2=-0.5$, 3 ch. & 530 & 0.3 & 1.1 & 0.7 & 0.0\\
Complete Journey & $\lambda_2=0.0$, 3 ch. & 531 & 0.3 & 0.4 & 0.0 & 0.0\\
Complete Journey & $\lambda_2=-0.5$, 2 ch. & 530 & 0.5 & 0.9 & 0.6 & 0.0\\
Complete Journey & $\lambda_2=0.0$, 2 ch. & 531 & 0.5 & 0.2 & -0.1 & -0.0\\
Complete Journey & $\lambda_2=0.5$, 2 ch. & 536 & 0.3 & 3.5 & 1.1 & 0.0\\
Ta Feng & $\lambda_2=-0.5$, 3 ch. & 19944 & 6.8 & 20.5 & 9.0 & -0.0\\
Ta Feng & $\lambda_2=0.0$, 3 ch. & 20215 & 5.2 & 13.4 & 5.8 & -0.0\\
Ta Feng & $\lambda_2=-0.5$, 2 ch. & 19944 & 10.5 & 15.4 & 7.1 & 0.0\\
Ta Feng & $\lambda_2=0.0$, 2 ch. & 20215 & 5.8 & 13.1 & 6.1 & -0.0\\
Ta Feng & $\lambda_2=0.5$, 2 ch. & 20576 & 6.4 & 70.7 & 25.2 & 0.2\\
Instacart & $\lambda_2=-0.5$, 3 ch. & 338 & 0.3 & 6.2 & 4.0 & 0.1\\
Instacart & $\lambda_2=0.0$, 3 ch. & 392 & 0.2 & 16.5 & 4.2 & 0.1\\
Instacart & $\lambda_2=-0.5$, 2 ch. & 338 & 4.7 & 1.7 & 0.8 & 0.0\\
Instacart & $\lambda_2=0.0$, 2 ch. & 392 & 0.9 & 16.3 & 4.0 & 0.1\\
Instacart & $\lambda_2=0.5$, 2 ch. & 474 & 0.1 & 53.8 & 17.1 & 0.4\\
\midrule
Complete Journey & margin low & 387 & 0.4 & 1.0 & 0.3 & -0.0\\
Complete Journey & margin high & 707 & 0.4 & 8.5 & 2.1 & 0.1\\
Complete Journey & margin Ta Feng tiers & 524 & 0.3 & 3.6 & 1.1 & 0.0\\
\midrule
Complete Journey & default (S1, $\gamma=1$, $N=50$k, C3) & 536 & 0.3 & 3.5 & 1.1 & 0.0\\
Ta Feng & default (S1, $\gamma=1$, $N=50$k, C3) & 20576 & 5.6 & 71.3 & 25.8 & 0.3\\
Instacart & default (S1, $\gamma=1$, $N=50$k, C3) & 474 & 0.1 & 53.4 & 17.0 & 0.4\\
\bottomrule
\end{tabular}

%% file: tables/empirical_more_data.tex
\begin{tabular}{llrrrrrrr}
\toprule
Data & Truth & Oracle & Exact regret & $-$Indep. & $-$Upg.-first & $-$Altern. & $-$Portfolio & $-$Margin-pop.\\
\midrule
Online Retail II & S1, $\gamma=1$ & 6615 & 4.0 & 22.3 & 5.5 & 0.0 & 0.0 & 68.2\\
Online Retail II & S2, $\gamma=1$ & 5992 & 6.3 & 28.3 & 9.6 & 0.2 & 0.2 & 91.6\\
Online Retail II & S3, $\gamma=1$ & 4663 & 893.9 & 31.8 & 3.7 & 0.0 & -0.0 & 56.6\\
Online Retail II & S1, $\gamma=0$ & 6575 & 4.2 & -0.2 & -0.1 & 0.0 & 0.0 & 45.1\\
Online Retail II & S1, $\gamma=2$ & 6693 & 5.3 & 47.5 & 13.6 & 0.2 & 0.2 & 96.5\\
Complete Journey & S1, Dominick's margins & 439 & 0.3 & 2.5 & 0.6 & 0.0 & 0.0 & 9.8\\
Complete Journey & S2, Dominick's margins & 415 & 1.9 & 3.9 & 0.7 & 0.0 & 0.0 & 9.7\\
Complete Journey & S3, Dominick's margins & 305 & 49.6 & 1.7 & 0.2 & -0.0 & -0.0 & 5.9\\
Carbo-Loading & S1, $\gamma=1$ & 226 & 0.2 & 1.7 & 0.7 & 0.0 & 0.0 & 10.6\\
Carbo-Loading & S2, $\gamma=1$ & 220 & 0.8 & 2.7 & 2.0 & 0.0 & 0.0 & 10.6\\
Carbo-Loading & S3, $\gamma=1$ & 163 & 31.5 & 1.3 & 1.3 & -0.0 & -0.0 & 6.3\\
Carbo-Loading & S1, $\gamma=0$ & 223 & 0.1 & -0.1 & -0.0 & 0.0 & 0.0 & 8.7\\
Carbo-Loading & S1, $\gamma=2$ & 231 & 0.1 & 6.1 & 3.4 & 0.0 & 0.0 & 15.0\\
Carbo-Loading & S1, $\gamma=1$, $\tau=1$ & 149 & 0.1 & 2.0 & 1.6 & 0.2 & 0.2 & 40.8\\
Carbo-Loading & S1, $\gamma=1$, $\tau=1.5$ & 120 & 0.1 & 1.4 & 1.0 & 0.1 & 0.1 & 64.7\\
Carbo-Loading & S1, $\gamma=1$, $\tau=2$ & 95 & 0.1 & 1.1 & 0.7 & 0.1 & 0.0 & 93.4\\
Carbo-Loading & S1, $\gamma=1$, item placebo & 226 & 0.2 & 0.1 & 0.1 & -0.0 & -0.0 & 8.9\\
Carbo-Loading & S1, $\gamma=1$, line placebo & 223 & 0.3 & -0.2 & -0.2 & 0.0 & 0.0 & 8.6\\
Carbo-Loading & S1, $\gamma=1$, Dominick's margins & 182 & 0.1 & 2.0 & 1.1 & 0.0 & 0.0 & 7.4\\
Carbo-Loading & S1, $\gamma=1$, EB & 226 & 0.2 & 1.7 & 0.7 & 0.0 & 0.0 & 10.6\\
\bottomrule
\end{tabular}

%% file: tables/empirical_coupling_share.tex
\begin{tabular}{llrrrrr}
\toprule
Data & Setting & Oracle & Upsell share & Coupling value & \multicolumn{2}{c}{Exact $-$ independent}\\
 & & per 1{,}000 & \% of oracle & per 1{,}000 (\% of oracle) & \% of oracle & \% of coupling value\\
\midrule
Complete Journey & S1 & 536 & 88 & 11.2 (2) & 0.65 & 31\\
Complete Journey & S2 & 509 & 84 & 28.4 (6) & 1.14 & 20\\
Ta Feng & S1 & 20576 & 88 & 600.8 (3) & 0.35 & 12\\
Ta Feng & S2 & 19523 & 85 & 1144.2 (6) & 0.53 & 9\\
Instacart & S1 & 474 & 33 & 159.2 (34) & 11.28 & 34\\
Instacart & S2 & 546 & 23 & 264.8 (48) & 20.23 & 42\\
Complete Journey & S1, $\gamma=0.5$ & 531 & 89 & 3.9 (1) & 0.21 & 29\\
Complete Journey & S1, $\gamma=2$ & 550 & 85 & 29.9 (5) & 1.65 & 30\\
Ta Feng & S1, $\gamma=0.5$ & 20283 & 90 & 277.9 (1) & 0.13 & 9\\
Ta Feng & S1, $\gamma=2$ & 21217 & 86 & 1319.6 (6) & 0.87 & 14\\
Instacart & S1, $\gamma=0.5$ & 396 & 40 & 75.0 (19) & 5.35 & 28\\
Instacart & S1, $\gamma=2$ & 634 & 24 & 325.5 (51) & 18.75 & 37\\
Online Retail II & S1, $\gamma=1$ & 6615 & 96 & 60.9 (1) & 0.34 & 37\\
Online Retail II & S2, $\gamma=1$ & 5992 & 95 & 138.9 (2) & 0.47 & 20\\
Online Retail II & S1, $\gamma=2$ & 6693 & 95 & 160.2 (2) & 0.71 & 30\\
Complete Journey & S1, Dominick's margins & 439 & 88 & 8.4 (2) & 0.58 & 30\\
Complete Journey & S2, Dominick's margins & 415 & 84 & 21.2 (5) & 0.93 & 18\\
Carbo-Loading & S1, $\gamma=1$ & 226 & 29 & 4.0 (2) & 0.73 & 41\\
Carbo-Loading & S2, $\gamma=1$ & 220 & 26 & 5.1 (2) & 1.22 & 52\\
Carbo-Loading & S1, $\gamma=2$ & 231 & 27 & 12.6 (5) & 2.63 & 48\\
Complete Journey & S1, EB shrinkage & 532 & 89 & 4.1 (1) & 0.17 & 22\\
Complete Journey & S2, EB shrinkage & 493 & 87 & 10.2 (2) & 0.44 & 21\\
Instacart & S1, EB shrinkage & 460 & 34 & 144.8 (31) & 10.19 & 32\\
Instacart & S2, EB shrinkage & 516 & 24 & 233.2 (45) & 18.09 & 40\\
\bottomrule
\end{tabular}

%% file: tables/empirical_e8.tex
\begin{tabular}{llrrrrrrr}
\toprule
Data & Setting & Oracle & Exact regret & $-$Indep. & $-$Greedy & $-$Altern. & $-$Portfolio & Differs\\
 & & per 1{,}000 & \multicolumn{5}{c}{\% of oracle margin} & \% baskets\\
\midrule
Complete Journey & Separate budgets (batch 1) & 536 & 0.06 & 0.65 & 0.08 & 0.00 & -- & 0.7\\
 & Display cost $\tau=0.5$ & 525 & 0.06 & 0.66 & 0.08 & 0.01 & 0.01 & 0.9\\
 & Display cost $\tau=1$ & 513 & 0.07 & 0.70 & 0.06 & 0.00 & 0.00 & 0.5\\
 & Display cost $\tau=2$ & 492 & 0.07 & 0.89 & 0.05 & 0.01 & 0.01 & 0.8\\
 & Budgets $L=K=3$ & 582 & 0.06 & 0.57 & 0.04 & 0.00 & 0.00 & 0.5\\
 & Up to 8 upsell groups & 625 & 0.08 & 0.64 & 0.12 & 0.01 & 0.01 & 0.8\\
 & 8 groups, $\tau=1$ & 602 & 0.08 & 0.70 & 0.11 & 0.01 & 0.01 & 0.7\\
 & Shared budget $T=3$ & 545 & 0.05 & 22.15 & 0.02 & 1.33 & 0.00 & 24.7\\
 & Shared budget $T=4$ & 571 & 0.05 & 6.59 & 0.02 & 0.96 & 0.00 & 21.0\\
 & Shared $T=4$, $\tau=1$ & 545 & 0.05 & 6.47 & 0.02 & 0.47 & -0.00 & 11.2\\
\midrule
Ta Feng & Separate budgets (batch 1) & 20576 & 0.03 & 0.35 & 0.05 & 0.00 & -- & 0.2\\
 & Display cost $\tau=0.5$ & 20096 & 0.03 & 0.36 & 0.05 & 0.00 & 0.00 & 0.3\\
 & Display cost $\tau=1$ & 19628 & 0.03 & 0.42 & 0.05 & 0.00 & 0.00 & 0.3\\
 & Display cost $\tau=2$ & 18766 & 0.03 & 0.76 & 0.04 & 0.01 & 0.01 & 0.8\\
 & Budgets $L=K=3$ & 22804 & 0.03 & 0.25 & 0.04 & 0.00 & 0.00 & 0.1\\
 & Up to 8 upsell groups & 22280 & 0.04 & 0.35 & 0.08 & 0.00 & 0.00 & 0.3\\
 & 8 groups, $\tau=1$ & 21329 & 0.04 & 0.41 & 0.08 & 0.00 & 0.00 & 0.4\\
 & Shared budget $T=3$ & 21276 & 0.02 & 25.29 & 0.02 & 1.52 & 0.00 & 20.9\\
 & Shared budget $T=4$ & 22435 & 0.02 & 8.62 & 0.01 & 1.24 & 0.00 & 24.1\\
 & Shared $T=4$, $\tau=1$ & 21432 & 0.02 & 8.81 & 0.01 & 0.77 & 0.00 & 16.8\\
\midrule
Instacart & Separate budgets (batch 1) & 474 & 0.02 & 11.28 & 2.80 & 0.08 & -- & 1.9\\
 & Display cost $\tau=0.5$ & 352 & 0.03 & 14.77 & 3.55 & 0.19 & 0.15 & 2.9\\
 & Display cost $\tau=1$ & 232 & 0.05 & 24.45 & 4.71 & 0.55 & 0.44 & 4.9\\
 & Display cost $\tau=2$ & 46 & 0.50 & 84.97 & 65.60 & 65.14 & 65.12 & 35.9\\
 & Budgets $L=K=3$ & 674 & 0.02 & 9.75 & 1.56 & 0.02 & 0.01 & 0.7\\
 & Up to 8 upsell groups & 482 & 0.03 & 11.88 & 3.22 & 0.12 & 0.09 & 2.6\\
 & 8 groups, $\tau=1$ & 240 & 0.06 & 25.17 & 5.49 & 0.60 & 0.48 & 5.6\\
 & Shared budget $T=3$ & 370 & 0.05 & 13.53 & 1.91 & 21.22 & 0.05 & 73.3\\
 & Shared budget $T=4$ & 500 & 0.03 & 16.16 & 1.63 & 22.92 & 0.03 & 67.9\\
 & Shared $T=4$, $\tau=1$ & 243 & 0.06 & 28.13 & 3.28 & 14.73 & 0.33 & 40.8\\
\bottomrule
\end{tabular}

%% file: tables/empirical_price_sensitivity.tex
\begin{tabular}{lrrrrrrrr}
\toprule
 & \multicolumn{4}{c}{Exact $-$ portfolio, \% of oracle} & \multicolumn{4}{c}{Exact regret, \% of oracle}\\
Price assumption & $\tau=0$ & $\tau=1$ & $\tau=1.5$ & $\tau=2$ & $\tau=0$ & $\tau=1$ & $\tau=1.5$ & $\tau=2$\\
\midrule
Organic $+25\%$ (main) & 0.08 & 0.44 & 16.13 & 65.12 & 0.02 & 0.05 & 0.14 & 0.50\\
Organic $+10\%$ & 0.11 & 3.02 & 70.58 & 89.31 & 0.02 & 0.09 & 0.54 & 3.31\\
Organic $+40\%$ & 0.04 & 0.13 & 1.53 & 5.78 & 0.03 & 0.04 & 0.09 & 0.16\\
Item price noise (s.d.\ 0.3) & 0.03 & 0.09 & 0.69 & 2.00 & 0.04 & 0.06 & 0.08 & 0.14\\
Margin rate 18\% & 0.05 & 0.44 & 16.13 & 65.12 & 0.02 & 0.05 & 0.14 & 0.50\\
Margin rate 38\% & 0.05 & 0.44 & 16.13 & 65.12 & 0.02 & 0.05 & 0.14 & 0.50\\
\bottomrule
\end{tabular}

%% file: tables/empirical_eb.tex
\begin{tabular}{llrrrr}
\toprule
Data & Setting & \multicolumn{2}{c}{Exact $-$ independent, per 1{,}000} & \multicolumn{2}{c}{Exact $-$ portfolio/altern., \% of oracle}\\
 & & fixed shrinkage & EB shrinkage & fixed & EB\\
\midrule
Complete Journey & S1 & 3.5 & 0.9 & 0.00 & 0.00\\
Complete Journey & S2 & 5.8 & 2.2 & 0.01 & 0.00\\
Complete Journey & S3 & 2.4 & 0.9 & 0.01 & -0.00\\
Complete Journey & S1, $\gamma=0.5$ & 1.1 & 0.2 & 0.00 & -0.00\\
Complete Journey & S1, $\gamma=2$ & 9.1 & 2.8 & 0.02 & 0.00\\
Complete Journey & S1, $\tau=2$ & 4.4 & 1.0 & 0.01 & 0.00\\
Complete Journey & line placebo & 1.5 & -0.0 & 0.00 & -0.00\\
Ta Feng & S1 & 71.3 & -1.2 & 0.00 & 0.00\\
Ta Feng & S2 & 104.0 & -1.1 & 0.01 & 0.00\\
Ta Feng & S3 & 46.0 & 4.7 & -0.00 & 0.00\\
Ta Feng & S1, $\gamma=0.5$ & 25.7 & -1.2 & 0.00 & 0.00\\
Ta Feng & S1, $\gamma=2$ & 185.3 & -1.2 & 0.01 & 0.00\\
Ta Feng & S1, $\tau=2$ & 143.2 & -1.5 & 0.01 & 0.00\\
Ta Feng & line placebo & 34.6 & -3.4 & -0.00 & 0.00\\
Instacart & S1 & 53.4 & 46.9 & 0.08 & 0.05\\
Instacart & S2 & 110.5 & 93.3 & 0.25 & 0.15\\
Instacart & S3 & 10.6 & 7.7 & 0.08 & 0.10\\
Instacart & S1, $\gamma=0.5$ & 21.2 & 18.3 & 0.03 & 0.02\\
Instacart & S1, $\gamma=2$ & 118.9 & 106.7 & 0.16 & 0.09\\
Instacart & S1, $\tau=2$ & 39.4 & 31.4 & 65.12 & 63.80\\
Instacart & line placebo & 6.6 & -0.1 & 0.02 & 0.00\\
\bottomrule
\end{tabular}

%% file: tables/empirical_robust.tex
\begin{tabular}{lrrr}
\toprule
 & Complete Journey & Ta Feng & Instacart\\
\midrule
\multicolumn{4}{l}{\emph{Bundle rebate: loss of rebate-blind exact, per 1{,}000 baskets}}\\
\quad $\rho=0.25$ & 0.0 & 0.8 & 1.5\\
\quad $\rho=0.5$ & 0.1 & 3.3 & 7.2\\
\quad $\rho=1.0$ & 0.6 & 19.9 & 34.0\\
\multicolumn{4}{l}{\emph{Stress test: exact regret, \% of oracle (base $\to$ rich model)}}\\
\quad cannibalization + fatigue & 17.5 $\to$ 0.10 & 17.3 $\to$ 0.08 & 10.6 $\to$ 0.08\\
\quad + test-period shift & 16.1 $\to$ 0.32 & 15.9 $\to$ 0.25 & 12.0 $\to$ 0.13\\
\multicolumn{4}{l}{\emph{Exact minus independent, per 1{,}000 baskets}}\\
\quad real premium labels & 3.5 & 71.3 & 53.4\\
\quad labels shuffled across items & 2.3 & 65.8 & 38.0\\
\quad labels random per purchase line & 1.5 & 34.6 & 6.6\\
\quad $\gamma=0$, ungated & -0.0 & -1.0 & -0.0\\
\quad $\gamma=0$, coupling gate & 0.0 & 0.0 & 0.0\\
\multicolumn{4}{l}{\emph{Exact regret by log size, \% of oracle}}\\
\quad $N=1$k & 0.21 & 0.09 & 0.30\\
\quad $N=2$k & 0.13 & 0.08 & 0.12\\
\quad $N=5$k & 0.08 & 0.06 & 0.06\\
\quad $N=10$k & 0.08 & 0.04 & 0.03\\
\quad $N=50$k & 0.06 & 0.03 & 0.02\\
\quad $N=200$k & 0.05 & 0.02 & 0.02\\
\bottomrule
\end{tabular}

%% file: tables/empirical_e7.tex
\begin{tabular}{lrrrrr}
\toprule
Data & Median $G$ & 90th pct.\ $G$ & Cyclic & Cyclic, $G\ge8$ & Width $\le2$\\
\midrule
Complete Journey & 6 & 22 & 50\% & 94\% & 66\%\\
Ta Feng & 4 & 12 & 22\% & 68\% & 93\%\\
Instacart & 4 & 8 & 68\% & 100\% & 54\%\\
\bottomrule
\end{tabular}

%% file: tables/empirical_scale.tex
\begin{tabular}{lrrrrrrrrr}
\toprule
Data & $L$ & $\tau$ & $G\ge8$ & DP = brute & DP time (median / p99) & DP faster & $-$Greedy & $-$Altern. & $-$Portfolio\\
 & & & \% baskets & checked & & \% baskets & \multicolumn{3}{c}{\% of oracle}\\
\midrule
Complete Journey & 2 & 0 & 34 & 887/887 & 3.2 ms / 33.2 s & 0.5 & 0.10 & 0.00 & 0.00\\
Complete Journey & 4 & 0 & 34 & 887/887 & 3.1 ms / 38.3 s & 27.1 & 0.09 & -0.00 & -0.00\\
Complete Journey & 4 & 1 & 34 & 887/887 & 3.0 ms / 37.8 s & 27.5 & 0.08 & 0.00 & -0.01\\
Ta Feng & 2 & 0 & 23 & 900/900 & 1.8 ms / 0.5 s & 0.6 & 0.09 & 0.01 & 0.01\\
Ta Feng & 4 & 0 & 23 & 900/900 & 2.1 ms / 0.5 s & 36.4 & 0.07 & 0.00 & 0.00\\
Ta Feng & 4 & 1 & 23 & 900/900 & 2.0 ms / 0.5 s & 35.8 & 0.06 & 0.00 & 0.00\\
Instacart & 2 & 0 & 12 & 893/893 & 3.5 ms / 7.6 s & 0.1 & 3.23 & 0.07 & 0.05\\
Instacart & 4 & 0 & 12 & 893/893 & 3.4 ms / 7.8 s & 4.5 & 2.81 & 0.04 & 0.03\\
Instacart & 4 & 1 & 12 & 893/893 & 3.1 ms / 7.9 s & 3.6 & 4.63 & 0.49 & 0.44\\
\bottomrule
\end{tabular}

%% file: tables/empirical_hybrid.tex
\begin{tabular}{llrrrrr}
\toprule
Budget & Method & Median & 95th pct. & 99th pct. & Worst & $<$100 ms\\
\midrule
$L=2$ & Always enumerate & 0.4 ms & 10.2 ms & 16.0 ms & 29.6 ms & 100.0\%\\
 & Always tree-decomposition DP & 2.5 ms & 1.1 s & 57.8 s & 132.6 s & 87.9\%\\
 & Hybrid (predicted faster method) & 0.4 ms & 10.2 ms & 16.0 ms & 29.6 ms & 100.0\%\\
 & Best of both per basket & 0.4 ms & 10.2 ms & 16.0 ms & 29.6 ms & 100.0\%\\
 & \textbf{Vectorized enumeration (E15c)} & 0.1 ms & 0.2 ms & 0.2 ms & 0.4 ms & 100.0\%\\
\midrule
$L=4$ & Always enumerate & 0.9 ms & 1.1 s & 2.0 s & 3.0 s & 85.8\%\\
 & Always tree-decomposition DP & 2.5 ms & 1.1 s & 68.3 s & 151.2 s & 87.6\%\\
 & Hybrid (predicted faster method) & 1.1 ms & 470.7 ms & 1.8 s & 3.0 s & 90.6\%\\
 & Best of both per basket & 0.9 ms & 470.7 ms & 1.8 s & 3.0 s & 90.7\%\\
 & \textbf{Vectorized enumeration (E15c)} & 0.1 ms & 16.1 ms & 29.5 ms & 110.6 ms & 99.97\%\\
\bottomrule
\end{tabular}

%% file: tables/empirical_prediction_quality.tex
\begin{tabular}{lrrrr}
\toprule
Predictor & Value RMSE & Oracle-policy & Margin regret & Regret per 1{,}000\\
 & \% oracle & agreement & \% oracle & baskets\\
\midrule
Oracle truth & 0.00 & 100.0\% & 0.000 & 0.00\\
Learned, $N=100$ & 17.20 & 54.8\% & 4.627 & 21.31\\
Learned, $N=250$ & 11.27 & 75.3\% & 1.106 & 5.09\\
Learned, $N=500$ & 6.34 & 83.9\% & 0.371 & 1.71\\
Learned, $N=1$k & 5.09 & 84.8\% & 0.299 & 1.38\\
Learned, $N=2$k & 3.76 & 90.7\% & 0.120 & 0.55\\
Learned, $N=5$k & 3.15 & 92.7\% & 0.059 & 0.27\\
Learned, $N=10$k & 2.63 & 95.2\% & 0.029 & 0.13\\
Learned, $N=50$k & 2.33 & 96.0\% & 0.018 & 0.08\\
Shuffled outcomes (50k) & 23.23 & 0.6\% & 38.394 & 176.78\\
\bottomrule
\end{tabular}

%% file: tables/empirical_market_basket.tex
\begin{tabular}{llrrrr}
\toprule
Premium anchor & Cross-sell receipt cell & Premium & Regular & Policy effect & Relative rate\\
\midrule
Pancake Mixes & syrups: Aunt Jemima & 11.3\% & 3.3\% & +8.0 pp & 3.43$\times$\\
Pasta Sauce & pasta: Barilla & 8.2\% & 4.8\% & +3.5 pp & 1.73$\times$\\
Pancake Mixes & syrups: Mrs Butterworth & 4.7\% & 1.7\% & +2.9 pp & 2.71$\times$\\
Pancake Mixes & syrups: Private Label & 9.4\% & 28.4\% & -15.0 pp & 0.33$\times$\\
Pasta Sauce & pasta: Private Label & 24.7\% & 37.5\% & -12.8 pp & 0.66$\times$\\
Pancake Mixes & pasta: Private Label & 8.5\% & 14.1\% & -5.6 pp & 0.60$\times$\\
\bottomrule
\end{tabular}

%% file: tables/empirical_mba_scaling.tex
\begin{tabular}{llrrr}
\toprule
Market & Quantity at $10^6$ baskets & $q=2$ & $q=3$ & $q=4$\\
\midrule
Instacart & Classical candidate itemsets & $1.1\times10^{9}$ & $1.7\times10^{13}$ & $2.0\times10^{17}$\\
 & Classical counting updates & $7.4\times10^{7}$ & $4.9\times10^{8}$ & $3.1\times10^{9}$\\
 & Typed contrasts estimated & 17{,}956 & 17{,}956 & 17{,}956\\
 & Typed counting updates & $7.6\times10^{7}$ & $7.6\times10^{7}$ & $7.6\times10^{7}$\\
\midrule
Complete Journey & Classical candidate itemsets & $2.3\times10^{9}$ & $5.1\times10^{13}$ & $8.5\times10^{17}$\\
 & Classical counting updates & $1.3\times10^{8}$ & $1.7\times10^{9}$ & $2.4\times10^{10}$\\
 & Typed contrasts estimated & 603{,}825 & 603{,}825 & 603{,}825\\
 & Typed counting updates & $1.5\times10^{8}$ & $1.5\times10^{8}$ & $1.5\times10^{8}$\\
\bottomrule
\end{tabular}

%% file: tables/empirical_power.tex
\begin{tabular}{llrrrr}
\toprule
Data & Setting & Comparison & Gap per 1{,}000 & Baskets/arm (offer log) & Baskets/arm (till)\\
\midrule
Complete Journey & default & vs independent & 3.9 & 1.4M & 67.1M\\
Complete Journey & default & vs upgrade-first & 1.4 & 11.3M & 550.2M\\
Complete Journey & default & vs alternating & 0.0 & 30943.0M & 1497174.2M\\
Complete Journey & default & vs no offers & 530.5 & 37 & 3,510\\
Complete Journey & shared budget $T=3$ & vs independent & 121.1 & 1,295 & 69k\\
Complete Journey & shared budget $T=3$ & vs upgrade-first & 9.1 & 266k & 12.2M\\
Complete Journey & shared budget $T=3$ & vs alternating & 7.0 & 444k & 20.4M\\
Complete Journey & shared budget $T=3$ & vs no offers & 540.9 & 38 & 3,379\\
Ta Feng & default & vs independent & 67.8 & 7.0M & 132.7M\\
Ta Feng & default & vs upgrade-first & 21.4 & 70.5M & 1330.6M\\
Ta Feng & default & vs alternating & 0.1 & 6499539.9M & 122662312.5M\\
Ta Feng & default & vs no offers & 19696.9 & 42 & 1,496\\
Ta Feng & shared budget $T=3$ & vs independent & 5238.0 & 1,076 & 22k\\
Ta Feng & shared budget $T=3$ & vs upgrade-first & 309.5 & 358k & 6.4M\\
Ta Feng & shared budget $T=3$ & vs alternating & 233.9 & 629k & 11.2M\\
Ta Feng & shared budget $T=3$ & vs no offers & 20400.8 & 42 & 1,400\\
Instacart & default & vs independent & 54.4 & 1,361 & 76k\\
Instacart & default & vs upgrade-first & 18.4 & 12k & 668k\\
Instacart & default & vs alternating & 0.3 & 53.0M & 2783.7M\\
Instacart & default & vs no offers & 476.1 & 9 & 972\\
Instacart & shared budget $T=3$ & vs independent & 53.1 & 1,089 & 78k\\
Instacart & shared budget $T=3$ & vs upgrade-first & 98.0 & 243 & 23k\\
Instacart & shared budget $T=3$ & vs alternating & 78.5 & 448 & 36k\\
Instacart & shared budget $T=3$ & vs no offers & 372.8 & 11 & 1,569\\
\bottomrule
\end{tabular}